\PassOptionsToPackage{numbers,sort&compress}{natbib} % Comment for named citations
\documentclass{article}

\usepackage[main,final]{neurips_2026}
\usepackage[american]{babel}
\usepackage{mathtools} % amsmath with fixes and additions
\usepackage{booktabs} % commands to create good-looking tables
\usepackage{tikz} % nice language for creating drawings and diagrams
\usepackage{amsmath}
\usepackage[colorinlistoftodos]{todonotes}
\usepackage{amssymb}
\usepackage{xifthen}
\usepackage{subcaption}
\usepackage{graphicx}
\graphicspath{{figures/}}
\usepackage{wrapfig}
\usepackage{textcomp}
\usepackage{yhmath}
\usepackage[ruled,vlined]{algorithm2e}
\usepackage{bm}
\usepackage{cancel}
\usepackage{xcolor}
\usepackage{pifont}
\usepackage{textcase}
\usepackage{siunitx}
\usepackage{array}
\usepackage{multirow}
\usepackage{gensymb}
\usepackage{tabularx}
\usepackage{extarrows}
\usepackage{colortbl}
\usepackage{pgfplots}
\usepackage{pgfplotstable}
\usepackage{adjustbox}
\usepackage{doi}

\usepackage[colorinlistoftodos]{todonotes}
\usepackage[utf8]{inputenc} % allow utf-8 input
\usepackage[T1]{fontenc}    % use 8-bit T1 fonts
\usepackage{csquotes}
\usepackage{hyperref}       % hyperlinks
\usepackage{url}            % simple URL typesetting
\usepackage{amsfonts}       % blackboard math symbols
\usepackage{amsthm}         % theorem environments
\newtheorem{definition}{Definition}[section]
\newtheorem{proposition}[definition]{Proposition}

\newtheorem{corollary}[definition]{Corollary}
\newtheorem{lemma}[definition]{Lemma}
\newtheorem*{remark}{Remark}
\newtheorem{assumption}{Assumption}[section]
\usepackage{nicefrac}       % compact symbols for 1/2, etc.
\usepackage{microtype}      % microtypography
\usepackage{authblk}

\pgfplotsset{compat=newest}
\usetikzlibrary{fadings}
\usetikzlibrary{bayesnet}
\usetikzlibrary{patterns}
\usetikzlibrary{shadows.blur}
\usetikzlibrary{calc}
\usetikzlibrary{shapes, 3d, backgrounds, fit, arrows.meta, positioning, shapes.geometric}
\usepackage{ifthen}
\usepackage{style/custom_macros}

\allowdisplaybreaks

\title{Bayesian Optimization with Fisher Information Geometry: Gradient Bounds and Trust-Region Methods}

\author{%
  Saksham Kiroriwal$^{1}$, Julius Pfrommer$^{1}$, and J\"{u}rgen Beyerer$^{1,2}$ \\
  $^{1}$Cognitive Industrial Systems, Fraunhofer IOSB, Karlsruhe, Germany \\
  $^{2}$Karlsruhe Institute of Technology (KIT), Karlsruhe, Germany \\
  \texttt{\{saksham.kiroriwal, julius.pfrommer, juergen.beyerer\}@iosb.fraunhofer.de}
}

\begin{document}

\maketitle

\begin{abstract}
 We study Bayesian optimization (BO) through the lens of information geometry. Pulling back the Fisher information metric through the surrogate posterior map yields a local sensitivity tensor on the input space, which leads to an upper bound on the gradient of reparameterizable acquisition functions (AFs). This view explains vanishing-gradient behavior in high-dimensional BO and provides a common interpretation of heuristics such as RAASP and dimension-scaled lengthscales. The upper bound can also serve as a generic recipe to understand the behavior of new acquisition functions or surrogates in high input dimensions as well. Using the bounds, we also identify necessary conditions for an input location to be an optimizer of the AF for a univariate Gaussian predictive posterior. Building on this analysis, we propose FITR, a trust-region-based BO method that replaces lengthscale-based scaling with local pullback-Fisher weights. FITR is not restricted to GP kernels with explicit lengthscales and easily generalizes to non-isotropic surrogates. FITR shows competitive performance on several benchmarks.
\end{abstract}

\section{Introduction}\label{sec:introduction}
\vspace{-1.5em}
\begin{figure}[h]
\centering
\begin{adjustbox}{max width=0.98\linewidth,max height=0.42\textheight,center}
\includegraphics{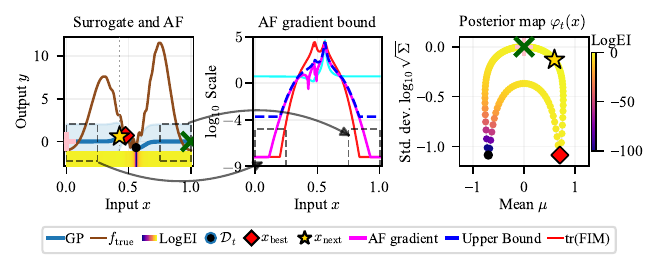}
\end{adjustbox}
\caption{Left: GP surrogate (SE kernel) and LogEI~\citep{ament2023unexpected} acquisition function. Middle: visualization of the upper bound for the acquisition function gradient from Proposition~\ref{prop:universal_bound}. Right: induced geometry on the input space by the posterior map in the statistical manifold defined by predictive mean and standard deviation. The color bar represents the same LogEI as in the left panel. Distant inputs can have similar predictions. Gray regions contain no observations and have a small pullback trace.
}
\label{fig:induced_geometry}
\end{figure}

Bayesian Optimization (BO) is a popular framework for optimizing black-box functions with expensive evaluations~\citep{frazier2018tutorial}. It has proven to be effective in hyperparameter tuning~\citep{wu2019hyperparameter}, robotics policy search~\citep{jaquier2020bayesian}, and engineering design~\citep{lam2018advances}. As explained in Section~\ref{sec:preliminaries}, BO alternates between fitting a probabilistic surrogate $\surrogatemodel_{\boiteration}$ to the observed data $\observeddata_{\boiteration}$ and optimizing an acquisition function (AF) $\acqf(\cdot)$ to select the next input candidate $\cusvector{\inpvars}_{\boiteration+1}$. In practice, the inner AF optimization is often non-trivial. The AF is usually non-convex, depends on the surrogate through a predictive distribution, and becomes increasingly difficult to optimize as the input dimension $\dims{\inpvars}$ grows.

This issue is now well documented in the high-dimensional BO (HDBO) literature. Recent work shows that both surrogate training and AF optimization can suffer from vanishing gradients in large dimensions~\citep{hvarfner2024vanilla, xu2024standard, papenmeier2025understanding}. Proposed fixes include dimensionally scaled lengthscale priors~\citep{hvarfner2024vanilla, xu2024standard}, Random Axis Aligned Subspace Perturbations (RAASP) initializations~\citep{papenmeier2022increasing}, conformal mappings~\citep{doumont2025we}, and trust-region-based search~\citep{eriksson2019scalable}. A central question remains: what common mechanism makes AF optimization collapse, and how much of that mechanism is tied to the surrogate model?

Inspired by the work of~\citet{herrmann2024position}, we analyze BO from a different perspective. As shown in Figure~\ref{fig:induced_geometry}, instead of looking at the AF only as a function over $\inpvarsset$, we utilize the chain rule and information geometry. The surrogate $\surrogatemodel_{\boiteration}$ is a map from input space $\inpvarsset$ to the statistical manifold $\manifold$ of predictive distributions. These predictive distribution parameters, like mean and variance, are the parameters that the AF really depends on. Information geometry~\citep{amari2016information, lee2018introduction} becomes a natural way to study how the surrogate induces geometry on the input space through the pullback Fisher tensor. This induced geometry tells us how strongly local perturbations in $\cusvector{\inpvars}$ change the predictive distribution.

This view separates two effects that are usually entangled in AF optimization. The acquisition design determines how sensitive the utility is to changes in the predictive distribution. The surrogate-induced geometry determines whether moving in input space produces such changes at all. Section~\ref{sec:manifold_view} formalizes this separation through an acquisition-gradient bound. This perspective then becomes a common analytical tool for understanding HDBO:
\begin{itemize}
\item In Section~\ref{sec:manifold_view}, we analyze the pullback geometry induced by the surrogate posterior map and derive a bound for the AF gradient norm $|\nabla_{\cusvector{\inpvars}}\acqf|_2$ that separates acquisition-side sensitivity from surrogate-induced geometry.
\item We use this factorization to analyze vanishing-gradient behavior in HDBO under a common framework. This yields a unified interpretation of dimensionally scaled priors~\citep{hvarfner2024vanilla, xu2024standard} and RAASP~\citep{papenmeier2022increasing, papenmeier2025understanding}.
\item In Section~\ref{sec:fitr}, we propose the Fisher-Information Trust Region (FITR), which replaces kernel-lengthscale scaling by local pullback-Fisher weights. FITR applies to differentiable probabilistic surrogates that satisfy the stated regularity conditions and is not restricted to GP kernels with explicit lengthscales. Section~\ref{sec:experiments} shows that FITR improves on TuRBO-REI on several benchmarks and remains competitive with strong baselines on the remainder.
\end{itemize}
\section{Preliminaries}
\label{sec:preliminaries}
% This section introduces the three building blocks that form the foundation of the results in Section~\ref{sec:manifold_view}: surrogate models as smooth maps into a statistical manifold, acquisition functions, and the Riemannian geometry that measures distances on that manifold. 
The aim of BO is to maximize a black-box function $\processfunction_{\mathrm{true}}$ over a feasible set $\inpvarsset \subset \mathbb{R}^{\dims{\inpvars}}$. At each BO iteration $\boiteration$, we assume that we have observed data $\observeddata_{\boiteration}=\{(\cusvector{\inpvars}_{\numobservation}, \obsouts_{\numobservation})\}_{\numobservation=1}^{\Numobservation}$, where input $\cusvector{\inpvars}_{\numobservation} \in \inpvarsset$ and output $\obsouts_{\numobservation} = \processfunction_{\mathrm{true}}(\cusvector{\inpvars}_{\numobservation}) + \noise_{\numobservation}$. We allow $\obsouts_{\numobservation}\in\mathbb{R}^{\multiobjective}$ to be multi-dimensional and assume $\noise_{\numobservation} \sim \mathcal{N}(0, \noisevar I_{\multiobjective})$, which reduces to $\mathcal{N}(0, \noisevar)$ in the scalar-output case. BO alternates between fitting a probabilistic surrogate $\surrogatemodel_{\boiteration}$ and maximizing an AF via $\arg \max_{\cusvector{\inpvars}\in \inpvarsset}\acqf(\cusvector{\inpvars};\observeddata_{\boiteration})$ to obtain the next input candidate $\cusvector{\inpvars}_{\boiteration+1}$~\citep{balandat2020botorch, garnett_bayesoptbook_2023}. We then evaluate $\obsouts_{\boiteration+1} = \processfunction_{\mathrm{true}}(\cusvector{\inpvars}_{\boiteration+1})$ and update the observed data $\observeddata_{\boiteration+1} = \observeddata_{\boiteration} \cup \{(\cusvector{\inpvars}_{\boiteration+1}, \obsouts_{\boiteration+1})\}$. We repeat until a stopping criterion is met.

\paragraph{Surrogate Models}
At each iteration $\boiteration$, the surrogate model $\surrogatemodel_{\boiteration}$ models the predictive distribution $\probability(\obsouts \vert \cusvector{\inpvars}, \observeddata_{\boiteration})=\probability(\obsouts; \cusvector[bm]{\manifoldparameters}_{\boiteration, \cusvector{\inpvars}})$, where $\cusvector[bm]{\manifoldparameters}_{\boiteration, \cusvector{\inpvars}}\in\manifold$ are the parameters of the predictive distribution~\citep{wilson2018maximizing}. The hyperparameters $\cusvector[bm]{\gphyperparam}_{\boiteration}$ of $\surrogatemodel_{\boiteration}$ specify the distribution parameters $\cusvector[bm]{\manifoldparameters}_{\boiteration, \cusvector{\inpvars}}\leftarrow\surrogatemodel_{\boiteration}(\cusvector{\inpvars};\observeddata_{\boiteration})$ and can be trained by maximizing the marginal log-likelihood~\cite{rasmussen2003gaussian} or another suitable objective~\cite{blei2017variational, moss2023inducing}.

As shown in the manifold view in Figure~\ref{fig:induced_geometry}, the surrogate model can also be interpreted as a smooth map from inputs to a family of probability distributions. Each input $\cusvector{\inpvars} \in \inpvarsset$ is mapped to a point $\cusvector[bm]{\manifoldparameters}_{\boiteration, \cusvector{\inpvars}}$ on the statistical manifold $\manifold$, the set of all distributions the surrogate can produce~\cite{amari2016information}. Input locations that are far apart in Euclidean space can map to nearby predictive parameters because the surrogate gives similar predictions there. This map is at least $C^1$ differentiable, a mild condition satisfied by many popular surrogate models such as Gaussian processes (GPs) with differentiable kernels~\cite{rasmussen2003gaussian} or Bayesian neural networks (BNNs) with smooth activations~\citep{li2023study}. The \emph{posterior map} $\manifoldmap_{\boiteration}$ induced by $\surrogatemodel_{\boiteration}$ can be defined as
\begin{equation}
    \label{eq:posterior_map}
    \manifoldmap_{\boiteration} : \inpvarsset \to \manifold, \quad \cusvector{\inpvars}
    \mapsto \cusvector[bm]{\manifoldparameters}_{\boiteration, \cusvector{\inpvars}},\text{ such that } \cusvector[bm]{\manifoldparameters}_{\boiteration, \cusvector{\inpvars}}\in\manifold,
\end{equation}
% where the Jacobian is $\jacobianmap_{\boiteration}(\cusvector{\inpvars}) = \partial\cusvector[bm]{\manifoldparameters}_{\cusvector{\inpvars},\boiteration}/\partial\cusvector{\inpvars} \in \mathbb{R}^{\dims{\manifoldparameters}\times\dims{\inpvars}}$. 
A common choice for the surrogate model is a Gaussian process (GP). For a scalar output $\obsouts\in\mathbb{R}$ at an input location $\cusvector{\inpvars}\in\inpvarsset$, $\surrogatemodel_{\boiteration}$ models $\cusvector[bm]{\manifoldparameters}_{\boiteration, \cusvector{\inpvars}}=(\distmean_{\boiteration, \cusvector{\inpvars}}, \distcov_{\boiteration, \cusvector{\inpvars}})$ and gives $\probability(\obsouts \vert \cusvector{\inpvars}, \observeddata_{\boiteration})=\mathcal{N}(\obsouts;\distmean_{\boiteration, \cusvector{\inpvars}}, \distcov_{\boiteration, \cusvector{\inpvars}})$ with $\manifold=\mathbb{R}\times\mathbb{R}_{+}$~\cite{rasmussen2003gaussian}. GPs are discussed in further detail in Appendix~\ref{apx:gaussian_process}.

\paragraph{Acquisition Function (AF)}
The AF $\acqf$ is the expected utility over the predictive posterior evaluated at a candidate location $\cusvector{\inpvars} \in \inpvarsset$ and can be defined as
% \begin{equation}
%     \label{eq:acqf_def}
%     \acqf(\cusvector{\inpvars}) = \mathbb{E}_{\obsouts\sim\probability(\obsouts \vert \cusvector{\inpvars}, \observeddata_{\boiteration})}\left[\acqfutility(\obsouts; \cusvector[bm]{\manifoldparameters}_{\boiteration, \cusvector{\inpvars}}, \cusvector[bm]{\acqfparams})\right],
% \end{equation}
$    \acqf(\cusvector{\inpvars}) = \mathbb{E}_{\obsouts}\left[\acqfutility(\obsouts; \cusvector[bm]{\manifoldparameters}_{\boiteration, \cusvector{\inpvars}}, \cusvector[bm]{\acqfparams})\right]$,
where $\acqfutility$ is the associated utility function for the AF and $\cusvector[bm]{\acqfparams}$ are respective acquisition parameters. The next query is $\cusvector{\inpvars}_{\boiteration+1} = \arg\max_{\cusvector{\inpvars}}\acqf(\cusvector{\inpvars};\observeddata_{\boiteration})$~\cite{mockus1998application, garnett_bayesoptbook_2023, wilson2018maximizing}.
Since $\probability(\obsouts \vert \cusvector{\inpvars}, \observeddata_{\boiteration})$ can be parameterized by $\cusvector[bm]{\manifoldparameters}_{\boiteration, \cusvector{\inpvars}}$, $\acqf$ depends on $\cusvector{\inpvars}$ only through $\cusvector[bm]{\manifoldparameters}_{\boiteration, \cusvector{\inpvars}}$. Using the reparameterization trick~\cite{wilson2017reparameterization}, we introduce a common notation to represent both closed-form and Monte Carlo AFs as
\begin{equation}
    \label{eq:acqf_expectation}
    \acqf(\cusvector{\inpvars}) = \mathbb{E}_{\cusvector[bm]{\MCbasesamples}}\!\left[\acqfdistfunctional\!\left(\cusvector[bm]{\manifoldparameters}_{\boiteration, \cusvector{\inpvars}},\cusvector[bm]{\MCbasesamples};\cusvector[bm]{\acqfparams}\right)\right],
\end{equation}
where $\cusvector[bm]{\MCbasesamples}$ is base noise drawn independently of $\cusvector{\inpvars}$~\cite{wilson2017reparameterization, wilson2018maximizing}. However for closed-form AFs, $\acqfdistfunctional$ is independent of $\cusvector[bm]{\MCbasesamples}$ and~\eqref{eq:acqf_expectation} can be trivially written as $\acqf(\cusvector{\inpvars})=\acqfdistfunctional(\cusvector[bm]{\manifoldparameters}_{\boiteration, \cusvector{\inpvars}};\cusvector[bm]{\acqfparams})$. Due to this, the closed-form AF and the $\cusvector[bm]{\MCbasesamples}$-dependent integrand used in the MC form can, in general, be different formulations of $\cusvector[bm]{\manifoldparameters}_{\boiteration, \cusvector{\inpvars}}$. Appendix~\ref{apx:acqf_table} discusses this in detail and lists both forms for different AFs.
% Following~\citet{wilson2017reparameterization}, EI using MC approximation is given by~\eqref{eq:ei_mc}.
% \begin{equation}
%     \label{eq:ei_mc}
%     \acqfdistfunctional_{\mathrm{EI}}(\cusvector[bm]{\manifoldparameters}_{\boiteration, \cusvector{\inpvars}},\, \cusvector[bm]{\MCbasesamples};\,\cusvector[bm]{\acqfparams})
%     = \max\!\left(\distmean_{\boiteration, \cusvector{\inpvars}} + \distcov_{\boiteration, \cusvector{\inpvars}}^{1/2}\cusvector[bm]{\MCbasesamples} - \obsouts^{\incumbentobs},\; 0\right).
% \end{equation}

% The gradient $\nabla_{\cusvector{\inpvars}}\acqf_{\mathrm{EI}}$ passes through $g$ and acts on $\cusvector[bm]{\manifoldparameters}_{\boiteration, \cusvector{\inpvars}}$, not on $\cusvector[bm]{\MCbasesamples}$ directly. For general exponential families, implicit reparameterization~\cite{balandat2020botorch} provides $g$.
% The analytic form~\eqref{eq:acqf_analytic} is the special case of~\eqref{eq:acqf_mc} where $\cusvector[bm]{\MCbasesamples}$ is trivial and $\acqfdistfunctional = \acqfdistfunctional$. Batch acquisitions (q-EI, q-UCB~\cite{balandat2020botorch}) and multi-objective acquisitions (qEHVI, qNEHVI~\cite{daulton2020differentiable}) all satisfy~\eqref{eq:acqf_mc} via their MC estimators. 
% In both cases, $\nabla_{\batchinpvars}\acqf$ is determined by the Jacobian of $\manifoldmap_{\boiteration}$; Section~\ref{sec:manifold_view} derives a universal bound on this gradient.

\paragraph{Fisher Information Matrix}
For a parametric distribution family $\probability(\obsouts; \cusvector[bm]{\manifoldparameters})$ with $\cusvector[bm]{\manifoldparameters} \in \manifold$, the \emph{Fisher information matrix} (FIM) $\fishermetric_\manifold(\cusvector[bm]{\manifoldparameters}) \in \mathbb{R}^{\dims{\manifoldparameters}\times\dims{\manifoldparameters}}$ is the expected outer product of $\nabla_{\cusvector[bm]{\manifoldparameters}} \log \probability(\obsouts; \cusvector[bm]{\manifoldparameters})$ as
\begin{equation}
    \label{eq:fim_def}
    \left[\fishermetric_\manifold(\cusvector[bm]{\manifoldparameters})\right]_{ij} =
    \mathbb{E}_{\probability(\obsouts;\,\cusvector[bm]{\manifoldparameters})}\!\left[
        \frac{\partial \log \probability(\obsouts;\cusvector[bm]{\manifoldparameters})}{\partial \manifoldparameters_i}\,
        \frac{\partial \log \probability(\obsouts;\cusvector[bm]{\manifoldparameters})}{\partial \manifoldparameters_j} \right].
\end{equation}
\citet{amari2000methods} showed that the FIM induces a Riemannian metric, known as the \emph{Fisher--Rao metric}, on the statistical manifold $\manifold$. The Fisher--Rao metric explains the local geometry of the manifold~\cite{amari2016information, martens2020new}. It quantifies the sensitivity of the predictive distribution $\probability(\obsouts; \cusvector[bm]{\manifoldparameters})$ to changes in $\cusvector[bm]{\manifoldparameters}$.
% At input $\cusvector{\inpvars}$, the posterior map yields $\cusvector[bm]{\manifoldparameters}_{\boiteration, \cusvector{\inpvars}} = \manifoldmap_{\boiteration}(\cusvector{\inpvars}) \in \manifold$, so $\fishermetric_\manifold(\cusvector[bm]{\manifoldparameters}_{\boiteration, \cusvector{\inpvars}})$ defines the local geometry of the predictive posterior $\probability(\obsouts \vert \cusvector{\inpvars}, \observeddata_{\boiteration})$ at that point on $\manifold$.

\paragraph{Riemannian Geometry}
A \emph{Riemannian manifold} $(\mathcal{M}, \pullbackFIM)$ is a smooth manifold $\mathcal{M}$ equipped with a Riemannian metric $\pullbackFIM$, whose pointwise evaluation $\pullbackFIM_p$ defines a positive definite inner product on each tangent space $T_p\mathcal{M} \simeq \mathbb{R}^{d_p}$~\cite{lee2018introduction}. The geodesic distance between $p_1, p_2 \in \mathcal{M}$ is the minimum path length on the manifold that connects the two points and is given by
\begin{equation}
    \label{eq:riemannian_distance}
    d_{\pullbackFIM}(p_1, p_2) = \inf_{\gamma} \int_0^1 \sqrt{\pullbackFIM_{\gamma(t)}\!\left(\dot\gamma(t),\, \dot\gamma(t)\right)}\,\mathrm{d}t,
\end{equation}
where the infimum is over smooth curves $\gamma : [0,1] \to \mathcal{M}$ with $\gamma(0)=p_1$ and $\gamma(1)=p_2$~\cite{lee2018introduction}. The Fisher--Rao metric $\fishermetric_\manifold(\cusvector[bm]{\manifoldparameters})$ induces a Riemannian structure on the statistical manifold $\manifold$~\cite{amari2016information}. We define the pullback Fisher tensor through the posterior map $\manifoldmap_{\boiteration}$ and discuss its properties in Section~\ref{sec:manifold_view}.
\section{Related Work}\label{sec:related_work}

\paragraph{Information geometry and applications.}
Natural gradient descent (NGD) rescales the ordinary gradient by the inverse Fisher metric to use the local sensitivity~\citep{amari1998natural, martens2020new}. Work by~\citet{khan2023bayesian} shows how algorithms such as stochastic gradient descent (SGD) and Kalman filters can be derived from an algorithm called \emph{Bayesian Learning Rule}, and how natural gradients can improve the resulting updates. As another application of information geometry, \citet{kim2022fisher} presented Fisher-Sharpness-Aware Minimization (FSAM), which replaced the Euclidean ball in sharpness-aware minimization with a Fisher ellipsoid, improving generalization and robustness. \citet{beik2021learning} learn Riemannian manifolds for geodesic motion generation in robotics.
% These works demonstrate the value of Fisher-Rao geometry in model parameter spaces. Whether such geometry can be induced on a black-box \emph{search} space---and whether it would provide useful structure for acquisition optimization---remains an open question.

\paragraph{Geometry-aware Bayesian optimization.}\citet{jaquier2020bayesian, jaquier2020high} introduced geometry-aware kernels that used geodesic distances in Bayesian optimization when the search space is non-Euclidean, such as robot orientation. However, in our work, we explore the geometry that is induced in the search space by the posterior map of the surrogate model $\surrogatemodel_{\boiteration}$ as mentioned in~\eqref{eq:posterior_map}. We do not assume that the search space has any intrinsic geometric structure. Recent work by~\citet{yu2025gitbo} proposes GIT-BO with Tabular Foundation Model (TabPFNv2) as the surrogate model. They use the gradient of the predictive mean to identify a global active subspace for high-dimensional BO. Our use of Fisher geometry is local and trust-region based: we shape candidate regions around the current center using the predictive distribution, not only the mean. A direct comparison with GIT-BO would therefore mix surrogate choice with trust-region geometry.

\paragraph{Pullback metrics for probabilistic models.}
\citet{tosi2014metrics} proposed the idea of pulling back a Riemannian metric for Gaussian Process Latent Variable Models (GPLVMs). They derived the expected metric tensor and used the resulting geodesics for improving data interpolation. \citet{arvanitidis2017latent} showed how the latent space can be characterized by a stochastic Riemannian metric, resulting in improved sampling and interpolation in latent space. \citet{arvanitidis2021pulling} used the Fisher-Rao metric associated with the space of decoder distributions, enabling meaningful latent geometries for a broader class of decoders. \citet{rozo2025riemann} generalized the construction to wrapped GPLVMs defined on Riemannian data manifolds, where they pull back a Riemannian structure from the data manifold to the latent space.
% The BO setting poses a complementary problem: the map runs from a high-dimensional search space to a low-dimensional predictive distribution family, and the goal is acquisition optimization rather than interpolation. The Fisher-Rao weights that make the metric sensitive to predictive uncertainty arise naturally in this reversed direction.

\paragraph{High-dimensional Bayesian optimization.}
Recent work by~\citet{hvarfner2024vanilla, xu2024standard} has shown that BO can be effective in high dimensions (HDBO). They analyze the vanishing-gradient problem in surrogate-model and acquisition-function (AF) optimization in high dimensions. They propose a dimensionally scaled lengthscale prior to avoid the vanishing gradient problem. \citet{papenmeier2025understanding} confirm the finding and propose $\sqrt{\dims{\inpvars}}/10$ initialization of lengthscales for Maximum Likelihood Estimation (MLE) for GP surrogates. They also show that RAASP for AF optimization helps in preventing vanishing gradients.
% These methods can also be seen as tools for improving other lengthscale-based kernel methods in high dimensions.

Local search for global optimization and trust-region methods offer complementary approaches. TuRBO is a popular trust-region-based method proposed by~\citet{eriksson2019scalable} to address high dimensionality. They use local axis-aligned hyperrectangles to define the trust region around the current best point. Work by~\citet{papenmeier2022increasing, papenmeier2023bounce} extends trust-region-based methods to combinatorial and mixed spaces as well. However, existing trust-region constructions in this setting rely on GP surrogates with explicit lengthscale parameters like the squared exponential kernel or Matern kernel.

\section{Manifold View of Bayesian Optimization}
\label{sec:manifold_view}

As mentioned in Section~\ref{sec:preliminaries}, the acquisition function depends on $\cusvector{\inpvars}$ through the surrogate's predictive parameters $\cusvector[bm]{\manifoldparameters}_{\boiteration, \cusvector{\inpvars}}$. Viewing the surrogate as a smooth map $\manifoldmap_{\boiteration}: \inpvarsset \to \manifold$ into the statistical manifold of predictive distributions shifts the analysis of acquisition-gradient behavior to the geometry of this map. Pulling back the Fisher--Rao metric from $\manifold$ to $\inpvarsset$ via $\manifoldmap_{\boiteration}$ yields the pullback Fisher tensor $\cusvector{\pullbackFIM}_{\boiteration}(\cusvector{\inpvars})$, which quantifies how sensitively $\probability(\obsouts ;\cusvector[bm]{\manifoldparameters}_{\boiteration, \cusvector{\inpvars}})$ changes under local perturbations of the input.

We define the pullback Fisher tensor and establish an acquisition gradient bound (Section~\ref{sec:universal_bound}). Then in Section~\ref{sec:sqrt_d} we show how the critical $\lengthscale$ scaling of $\sqrt{\dims{\inpvars}}$ from~\cite{hvarfner2024vanilla, papenmeier2025understanding, xu2024standard} and other HDBO findings like RAASP can be viewed from the geometric perspective of the pullback Fisher tensor.
% The main result (Proposition~\ref{prop:universal_bound}) establishes that the acquisition gradient norm is bounded by the product of $\mathrm{tr}(\pullbackFIM_{\boiteration})$ and an acquisition sensitivity that depends only on the functional form of the acquisition function. Gradient vanishing in high dimensions therefore traces to the geometry of the posterior map, not to any particular choice of acquisition.

\begin{assumption}[Surrogate Differentiability and Positive Fisher Metric]
    \label{ass:surrogate_regularity}
    $\manifoldmap_{\boiteration}$ is $C^1$ on $\inpvarsset$ (satisfied by GPs with differentiable kernels and by BNNs with smooth activations) and $\fishermetric_{\manifold}\!\left(\manifoldmap_{\boiteration}(\cusvector{\inpvars})\right) \succ 0$ for all $\cusvector{\inpvars}\in\inpvarsset$ (a minimum noise $\noisevar > 0$ ensures this for GP regression).
\end{assumption}

As discussed in Appendix~\ref{apx:acqf_regularity}, the first condition makes $\manifoldmap_{\boiteration}$ differentiable so that the Jacobian $\jacobianmap_{\boiteration}$ exists. The second ensures $\fishermetric_\manifold$ is a valid inner product on $\manifold$, so its inverse (required in the acquisition sensitivity~\eqref{eq:acqf_sensitivity}) is well-defined. The pullback Fisher tensor quantifies the sensitivity of the predictive distribution to perturbations in the input space at each input location.

\begin{definition}[Pullback Fisher Tensor]
    \label{def:pullback_tensor}
    Under Assumption~\ref{ass:surrogate_regularity}, the pullback Fisher tensor induced on the input space $\inpvarsset$ by the predictive distribution manifold $\manifold$ through the posterior map $\manifoldmap_{\boiteration}$ is

    \begin{equation}
        \label{eq:pullback_tensor}
        \cusvector{\pullbackFIM}_{\boiteration}(\cusvector{\inpvars}) = \jacobianmap_{\boiteration}(\cusvector{\inpvars})^\top\,
        \fishermetric_{\manifold}\!\bigl(\manifoldmap_{\boiteration}(\cusvector{\inpvars})\bigr)\,
        \jacobianmap_{\boiteration}(\cusvector{\inpvars})
        \;\in\; \mathbb{R}^{\dims{\inpvars}\times\dims{\inpvars}},
    \end{equation}
    where $\cusvector{\jacobianmap}_{\boiteration}(\cusvector{\inpvars}) = \partial\cusvector[bm]{\manifoldparameters}_{\boiteration, \cusvector{\inpvars}}/\partial\cusvector{\inpvars} \in \mathbb{R}^{\dims{\manifoldparameters}\times\dims{\inpvars}}$ is the Jacobian of the posterior map.
\end{definition}
Figure~\ref{fig:induced_geometry} shows the induced geometry on a 1D Euclidean input space by the GP surrogate with SE kernel and LogEI AF using the pullback metric tensor. Figure~\ref{fig:induced_geometry_2d} similarly shows the trace of the pullback metric tensor for a 2D input space. Since for $\dims{\inpvars}>1$, the pullback metric is a tensor, the eigenvalues of $\cusvector{\pullbackFIM}_{\boiteration}(\cusvector{\inpvars})$ reflect how strongly the predictive distribution changes with respect to perturbations in input space at $\cusvector{\inpvars}$ in respective directions. Large eigenvalues correspond to directions (not necessarily axes-aligned with $\inpvarsset$) along which the distribution changes rapidly, small or zero eigenvalues to directions that are locally uninformative. This effect can be seen in the "Conventional BO view" panel of Figure~\ref{fig:induced_geometry_2d}, where we show the gradient $\nabla_{\cusvector{\inpvars}}\acqf(\cusvector{\inpvars})$ and the natural gradient $\cusvector{\pullbackFIM}_{\boiteration}(\cusvector{\inpvars})^{-1}\nabla_{\cusvector{\inpvars}}\acqf(\cusvector{\inpvars})$ (where the inverse exists) vector field and also zoomed in view for two input locations denoted by the magenta and red boxes.

The pullback tensor $\cusvector{\pullbackFIM}_{\boiteration}(\cusvector{\inpvars})$ is a Riemannian metric only when it is positive definite. Its rank is at most $\dims{\manifoldparameters}$ (at most $2$ for a scalar Gaussian predictive), so the induced geometry is degenerate when $\dims{\inpvars}>\dims{\manifoldparameters}$. Locally, perturbations at $\cusvector{\inpvars}$ in null directions leave the predictive parameters $\cusvector[bm]{\manifoldparameters}_{\cusvector{\inpvars}}$ unchanged.

\begin{figure}[t]
    \centering
    \begin{adjustbox}{max width=0.98\linewidth,max height=0.42\textheight,center}
        \includegraphics{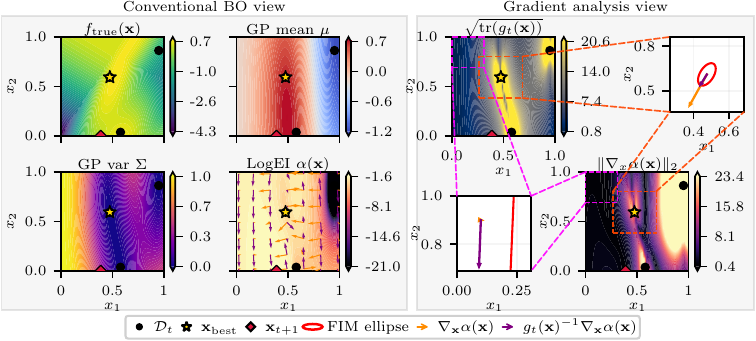}
    \end{adjustbox}
    \caption{Conventional view of BO with a GP surrogate (SE kernel) and LogEI AF. The LogEI plot also shows the gradient (orange) and natural gradient (purple) vector field in the input space. The "gradient analysis" shows pullback metric trace $\sqrt{\mathrm{tr}(\cusvector{\pullbackFIM}_{\boiteration}(\cusvector{\inpvars}))}$ and AF gradient norm $\|\nabla_{\cusvector{\inpvars}}\acqf(\cusvector{\inpvars})\|_2$. Trace is non-degenerate near the observed data and collapses in regions where the GP posterior has reverted to the prior (no observations). The red ellipses are calculated using the pullback metric tensor at respective input locations.}
    \label{fig:induced_geometry_2d}
    \vspace{-1.0em}
\end{figure}

\subsection{Acquisition Gradient Bound via Pullback Fisher Tensor}
\label{sec:universal_bound}

\begin{proposition}[Acquisition Gradient Bound]
    \label{prop:universal_bound}
    Let Assumption~\ref{ass:surrogate_regularity} hold. For any $\cusvector{\inpvars} \in \inpvarsset$, let $\cusvector[bm]{\manifoldparameters}_{\boiteration, \cusvector{\inpvars}}=\manifoldmap_{\boiteration}(\cusvector{\inpvars})$ and suppose the acquisition admits the representation
    $\acqf(\cusvector{\inpvars}) = \mathbb{E}_{\cusvector[bm]{\MCbasesamples}}[\acqfdistfunctional(\cusvector[bm]{\manifoldparameters}_{\boiteration, \cusvector{\inpvars}}, \cusvector[bm]{\MCbasesamples})]$,
    where the interchange of gradient and expectation is valid, for example under the scalar Gaussian reparameterization conditions of Appendix~\ref{apx:acqf_regularity}, and $\mathbb{E}_{\cusvector[bm]{\MCbasesamples}}[\|\nabla_{\cusvector[bm]{\manifoldparameters}_{\boiteration, \cusvector{\inpvars}}}\acqfdistfunctional\|^2] < \infty$. Then
    \begin{equation}
        \label{eq:universal_bound}
        \big\|\nabla_{\cusvector{\inpvars}}\,\acqf(\cusvector{\inpvars})\big\|_2^2
        \;\leq\;
        \acqfsensitivity\!\bigl(\cusvector[bm]{\manifoldparameters}_{\boiteration, \cusvector{\inpvars}}\bigr)
            \cdot \mathrm{tr}\!\left(\cusvector{\pullbackFIM}_{\boiteration}(\cusvector{\inpvars})\right),
    \end{equation}
    where the \emph{acquisition sensitivity} is
    \begin{equation}
        \label{eq:acqf_sensitivity}
        \acqfsensitivity(\cusvector[bm]{\manifoldparameters}_{\boiteration, \cusvector{\inpvars}}) =
        \mathbb{E}_{\cusvector[bm]{\MCbasesamples}}\!\left[\nabla_{\cusvector[bm]{\manifoldparameters}_{\boiteration, \cusvector{\inpvars}}}\acqfdistfunctional\!\left(\cusvector[bm]{\manifoldparameters}_{\boiteration, \cusvector{\inpvars}}, \cusvector[bm]{\MCbasesamples}\right)^\top
        {\fishermetric_{\manifold}\!\left(\cusvector[bm]{\manifoldparameters}_{\boiteration, \cusvector{\inpvars}}\right)}^{-1}\right.
        \left.\nabla_{\cusvector[bm]{\manifoldparameters}_{\boiteration, \cusvector{\inpvars}}}\acqfdistfunctional\!\left(\cusvector[bm]{\manifoldparameters}_{\boiteration, \cusvector{\inpvars}}, \cusvector[bm]{\MCbasesamples}\right)\right].
    \end{equation}
\end{proposition}
Proof is given in Appendix~\ref{apx:proof_universal}. \autoref{eq:universal_bound} separates the roles of the surrogate and the acquisition in the inner BO optimization problem. The trace $\mathrm{tr}(\cusvector{\pullbackFIM}_{\boiteration}(\cusvector{\inpvars}))$ depends only on the surrogate map and measures how strongly the predictive distribution changes under local perturbations at $\cusvector{\inpvars}$. \autoref{eq:acqf_sensitivity} defines the \emph{acquisition-side sensitivity} of the AF value to perturbations on the predictive manifold $\manifold$. For a fixed AF and predictive distribution family, the formula for acquisition-side sensitivity depends on the current predictive state $\cusvector[bm]{\manifoldparameters}$ and the manifold metric $\fishermetric_{\manifold}(\cusvector[bm]{\manifoldparameters})$, but not on which surrogate produced that state. Surrogates represented through the same Gaussian predictive family, such as GP posteriors and Gaussian predictive approximations to Bayesian neural networks (BNNs), therefore use the same acquisition-side formula when evaluated at the same predictive parameters. The surrogate-specific effect enters through $\mathrm{tr}(\cusvector{\pullbackFIM}_{\boiteration}(\cusvector{\inpvars}))$. Appendix~\ref{apx:gaussian_sensitivity} discusses $\acqfsensitivity$ for common AFs under a Gaussian predictive distribution.

When the surrogate maps a region of input space to a near-constant predictive distribution, or when the posterior reverts to the prior away from the data, $\mathrm{tr}(\cusvector{\pullbackFIM}_{\boiteration}(\cdot))$ collapses in that region. Proposition~\ref{prop:universal_bound} then shows that bounded acquisition-side sensitivity forces the acquisition gradient to become small there. This effect is visible in Figure~\ref{fig:induced_geometry} and Figure~\ref{fig:induced_geometry_2d}: the AF gradient vanishes where the pullback trace collapses and remains active where the local predictive geometry is non-degenerate. Table~\ref{tab:sensitivity_bounds} gives acquisition-specific constants for Gaussian predictives.

\begin{table}[t]
    \centering
    \caption{AF constants $\acqfsensitivity(\cusvector[bm]{\manifoldparameters})$ for scalar Gaussian predictive distribution $\manifoldparameters\in\manifold$. Derivations are in Appendix~\ref{apx:gaussian_sensitivity}, including Lemma~\ref{lem:mc_sensitivity_bound}.}
    \label{tab:sensitivity_bounds}
    \small
    \begin{tabular}{lcc}
        \toprule
        AF & Analytic representation & Monte Carlo representation \\
        \midrule
        EI & $\leq \distcov(1+1/(4\pi))$ & $\leq \tfrac32\distcov$ \\
        UCB & $= \distcov(1+\ucbexploration/2)$ & $= \distcov(1+\ucbexploration\pi/4)$ \\
        PI & $\leq 1/(2\pi)$ & $\leq 3\distcov/(32\pitemperature^2)$ \\
        LogEI & $= \dfrac{\acqfsensitivity^{\mathrm{EI}}(\cusvector[bm]{\manifoldparameters})}
        {\acqfdistfunctional^{\mathrm{EI}}(\cusvector[bm]{\manifoldparameters})^2}$ & -- \\
        \bottomrule
    \end{tabular}
\end{table}

\paragraph{Gradient Bound using Eigenvalues and Geometric Interpretation at Stationary Points}
For a BO iteration $\boiteration$ and an input $\cusvector{\inpvars}$, with predictive parameters $\cusvector[bm]{\manifoldparameters}_{\boiteration,\cusvector{\inpvars}}=\manifoldmap_{\boiteration}(\cusvector{\inpvars})$, the proof in Appendix~\ref{apx:proof_universal} formulates the AF gradient as $\nabla_{\cusvector{\inpvars}}\acqf(\cusvector{\inpvars})=\cusmatrix{\fisherjacobian}^{\top}\cusvector{\fisheracqgradient}$, where
\begin{equation}
    \label{eq:spectral_factors}
        \cusmatrix{\fisherjacobian}
        =\fishermetric_{\manifold}\!\left(\cusvector[bm]{\manifoldparameters}_{\boiteration,\cusvector{\inpvars}}\right)^{1/2}\jacobianmap_{\boiteration}(\cusvector{\inpvars}),\qquad
        \cusvector{\fisheracqgradient}
        =\fishermetric_{\manifold}\!\left(\cusvector[bm]{\manifoldparameters}_{\boiteration,\cusvector{\inpvars}}\right)^{-1/2}
        \expectation_{\cusvector[bm]{\MCbasesamples}}
        \!\left[\nabla_{\cusvector[bm]{\manifoldparameters}_{\boiteration,\cusvector{\inpvars}}}
        \acqfdistfunctional\!\left(\cusvector[bm]{\manifoldparameters}_{\boiteration,\cusvector{\inpvars}},\cusvector[bm]{\MCbasesamples}\right)\right].
\end{equation}
Let $\cusmatrix{\fishercompanion}=\cusmatrix{\fisherjacobian}\cusmatrix{\fisherjacobian}^{\top}$, and $\fishereigenvalue_{\min}$ and $\fishereigenvalue_{\max}$ denote the smallest and largest eigenvalues of $\cusmatrix{\fishercompanion}$. One can also write $\|\nabla_{\cusvector{\inpvars}}\acqf(\cusvector{\inpvars})\|_2^2=\cusvector{\fisheracqgradient}^{\top}\cusmatrix{\fishercompanion}\cusvector{\fisheracqgradient}$. For a scalar Gaussian predictive, $\cusmatrix{\fishercompanion}$ is $2\times2$, the matrix $\cusmatrix{\fishercompanion}$ is a Hermitian matrix, so using Rayleigh quotient we can also show that the acquisition gradient norm is bounded by the eigenvalues of $\cusmatrix{\fishercompanion}$ as
\begin{equation}
    \label{eq:spectral_bracket}
        \fishereigenvalue_{\min}\|\cusvector{\fisheracqgradient}\|_2^2
        \;\leq\; \|\nabla_{\cusvector{\inpvars}}\acqf(\cusvector{\inpvars})\|_2^2
        \;\leq\; \fishereigenvalue_{\max}\|\cusvector{\fisheracqgradient}\|_2^2
        \;\leq\; \acqfsensitivity\!\left(\cusvector[bm]{\manifoldparameters}_{\boiteration,\cusvector{\inpvars}}\right)
        \mathrm{tr}\!\left(\cusvector{\pullbackFIM}_{\boiteration}(\cusvector{\inpvars})\right).
\end{equation}
For a scalar Gaussian predictive, $\cusmatrix{\fishercompanion}=\cusmatrix{\fisherjacobian}\cusmatrix{\fisherjacobian}^{\top}$ and the pullback tensor, $\cusvector{\pullbackFIM}_{\boiteration}(\cusvector{\inpvars})=\cusmatrix{\fisherjacobian}^{\top}\cusmatrix{\fisherjacobian}$, has the same trace. Since $\fishereigenvalue_{\max}\leq\mathrm{tr}(\cusmatrix{\fishercompanion})\leq2\fishereigenvalue_{\max}$, we can say that the trace upper bounds has a slack of at most $\sqrt{2}$ times the largest eigenvalue, with the same acquisition factor, independent of the input dimension. As shown in the Appendix~\ref{apx:proof_universal}, $\|\cusvector{\fisheracqgradient}\|_2^2\leq\acqfsensitivity(\cusvector[bm]{\manifoldparameters}_{\boiteration,\cusvector{\inpvars}})$, where the equality exists for the analytic acquisition functions.
For a scalar Gaussian predictive with variance $\distcov_{\boiteration,\cusvector{\inpvars}}>0$ and nonzero mean and variance gradients, the two rows of $\cusmatrix{\fisherjacobian}$ in~\eqref{eq:spectral_factors} are $(\nabla_{\cusvector{\inpvars}}\distmean_{\boiteration,\cusvector{\inpvars}})^\top/\sqrt{\distcov_{\boiteration,\cusvector{\inpvars}}}$ and $(\nabla_{\cusvector{\inpvars}}\distcov_{\boiteration,\cusvector{\inpvars}})^\top/(\sqrt{2}\distcov_{\boiteration,\cusvector{\inpvars}})$. 

At an interior stationary point, the gradient $\nabla_{\cusvector{\inpvars}}\acqf(\cusvector{\inpvars})=\cusmatrix{\fisherjacobian}^{\top}\cusvector{\fisheracqgradient}=0$, which implies $\cusmatrix{\fishercompanion}\cusvector{\fisheracqgradient}=0$. If the acquisition-side factor $\cusvector{\fisheracqgradient}\neq\cusvector{0}$, this implies that $\cusmatrix{\fishercompanion}$ is singular and has a nonzero nullspace. This can happen only if the mean and variance gradients are collinear, as shown in Figure~\ref{fig:theoretical_insights}(a). That is, at an interior stationary point, which may be a saddle point or a local or global optimum, $\nabla_{\cusvector{\inpvars}}\distmean_{\boiteration,\cusvector{\inpvars}}=c\nabla_{\cusvector{\inpvars}}\distcov_{\boiteration,\cusvector{\inpvars}}$, where $c\in\mathbb{R}$ is determined by the acquisition derivatives at that point. For an analytic AF with nonzero mean derivative, $c=-\frac{\partial\acqfdistfunctional(\cusvector[bm]{\manifoldparameters}_{\boiteration,\cusvector{\inpvars}})/\partial\distcov_{\boiteration,\cusvector{\inpvars}}}{\partial\acqfdistfunctional(\cusvector[bm]{\manifoldparameters}_{\boiteration,\cusvector{\inpvars}})/\partial\distmean_{\boiteration,\cusvector{\inpvars}}}$. This collinearity condition is necessary, not sufficient, for stationarity. This is an important insight as the inequality in~\eqref{eq:spectral_bracket} defines the necessary conditions for an input location $\cusvector{\inpvars}$ to be an optimizer of AF for univariate Gaussian predictive posterior. This can motivate new optimization methods.
% 
% As explained in Appendix~\ref{apx:acqf_regularity}, the regularity conditions are satisfied by various acquisition functions like EI, Upper Confidence Bound (UCB), Probability of Improvement (PI), etc. Under Assumption~\ref{ass:surrogate_regularity}, $\fishermetric_{\manifold}(\cusvector[bm]{\manifoldparameters}_{\boiteration, \cusvector{\inpvars}}) \succ 0$ forces $\cusvector{\distcov}_{\boiteration}(\cusvector{\inpvars}) \succ 0$, which prevents the $\nabla\cusvector{\distcov}_{\boiteration}^{1/2} \to \infty$ hence keeping the $\nabla_{\cusvector[bm]{\manifoldparameters}_{\boiteration, \cusvector{\inpvars}}}\acqfdistfunctional$ square-integrable.

% -------------------------------------------------------------

\begin{figure}[t]
    \centering
    \includegraphics[width=\linewidth]{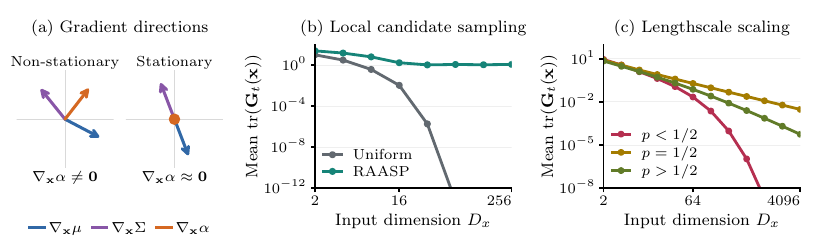}
    \caption{Unified view of HDBO. (a) Collinear mean and variance gradients at a stationary point with $\nabla_{\cusvector{\inpvars}}\acqf\approx0$. (b) RAASP candidates have nonzero trace of pullback tensor $\mathrm{tr}\!\left(\cusvector{\pullbackFIM}_{\boiteration}(\cusvector{\inpvars})\right)$. (c) $\lengthscaleexponent=1/2$ scaling gives the slowest asymptotic decay of $\mathrm{tr}\!\left(\cusvector{\pullbackFIM}_{\boiteration}(\cusvector{\inpvars})\right)$.}
    \label{fig:theoretical_insights}
\end{figure}

\subsection{Unified FIM-based Framework for HDBO analysis}
\label{sec:sqrt_d}
\label{sec:fim_explanations}

\paragraph{Lengthscale Scaling and Random Axis Aligned Subspace Perturbations}
For GP surrogates with AFs such as EI, UCB, or PI, the sensitivity $\acqfsensitivity(\manifoldmap_{\boiteration}(\cusvector{\inpvars}))$ remains uniformly bounded as $\dims{\inpvars}$ grows, as shown in Appendix~\ref{apx:gaussian_sensitivity}. The experiments use LogEI, for which the same local regularity argument applies in the strictly positive-variance regime considered here.
% Using Proposition~\ref{prop:universal_bound}, it can be deduced that for bounded $\acqfsensitivity(\manifoldmap_{\boiteration}(\cusvector{\inpvars}))$, if the pullback trace $\mathrm{tr}(\pullbackFIM_{\boiteration}(\cusvector{\inpvars}))$ collapses, then the AF gradient also vanishes. As shown in Figure~\ref{fig:induced_geometry} and Figure~\ref{fig:induced_geometry_2d}, the pullback trace collapses in regions where the GP posterior has reverted to the prior (no observations) and the AF gradient vanishes in the same regions. The visualization also explains both the $\sqrt{\dims{\inpvars}}$ lengthscale scaling~\cite{xu2024standard,hvarfner2024vanilla} and the effectiveness of RAASP~\citep{papenmeier2022increasing}.
As shown in Appendix~\ref{apx:proof_hd}, for i.i.d.\ inputs drawn uniformly from $[0,1]^{\dims{\inpvars}}$, the expected Euclidean distance $d=\|\cusvector{\inpvars}-\cusvector{\inpvars}'\|$ is $\Theta(\sqrt{\dims{\inpvars}})$. We can write $\lengthscale\propto\dims{\inpvars}^{\lengthscaleexponent}$, where $\lengthscaleexponent$ controls the lengthscale growth. The SE kernel then satisfies $\gpkernel_\lengthscale(\cusvector{\inpvars},\cusvector{\inpvars}')=\signalvar\exp[-\Theta(\dims{\inpvars}^{1-2\lengthscaleexponent})]$. The $\sqrt{\dims{\inpvars}}$ scaling was identified by~\citet{hvarfner2024vanilla,xu2024standard,papenmeier2025understanding}. The manifold view also explains why $\sqrt{\dims{\inpvars}}$ scaling is the optimal dimension informed scaling for lengthscale priors. As shown in Figure~\ref{fig:theoretical_insights}(c) and Appendix~\ref{apx:proof_hd}, $\lengthscaleexponent=1/2$ gives the slowest asymptotic trace decay, $\Theta(\dims{\inpvars}^{-1})$, among other scalings. Isotropic Mat\'{e}rn kernels depend on the same normalized distance $d/\lengthscale$ and have the same critical scaling. The trace-collapse argument also helps explain why RAASP around observed data~\citep{papenmeier2022increasing} improves multi-start AF optimization (Figure~\ref{fig:theoretical_insights}(b)). The pullback Fisher tensor therefore provides a common language for both failure modes and fixes in HDBO.

\paragraph{TuRBO connection}
Figure~\ref{fig:induced_geometry_2d} shows that the pullback trace and the AF gradient remain non-degenerate near the observed data. TuRBO develops on a similar principle, where it centers a local axis-aligned TR at the incumbent and adapts its size over time using the learned lengthscales~\cite{eriksson2019scalable}. The dimensions with longer lengthscales have a longer TR length. Appendix~\ref{apx:proof_turbo_fitr} shows that the metric tensor at a local optimum $\cusvector{\inpvars}_{\mathrm{best}}$ is often dominated by the variance-gradient term, which is inversely proportional to the lengthscales. The ellipsoid made using the pullback metric would be longer in the dimensions with longer lengthscales. This is similar to how TuRBO uses the learned lengthscales to create TR. This motivates the use of pullback-metric-based TR for BO, which we call the Fisher-Information Trust Region (FITR). The FITR is a location-dependent TR that can be extended to any surrogate for which the derivative of the predictive distribution with respect to the input variables can be calculated. Figure~\ref{fig:fitr_ellipses} in Appendix~\ref{apx:implementation_details} illustrates the local geometry for different surrogates.

\section{Fisher-Information Trust Region (FITR)}
\label{sec:fitr}

\begin{wrapfigure}[16]{r}{0.48\textwidth}
    \vspace{-\intextsep}
    \begin{minipage}{0.48\textwidth}
        \small
        \begin{algorithm}[H]
            \caption{FITR weight construction}
            \label{alg:fitr_weights}
            \KwIn{Center $\cusvector{\inpvars}_{\boiteration, c}$, data $\observeddata_{\boiteration}$}
            \KwOut{Raw FITR weights $\TRlengthweight_{\boiteration,j}$ for $j=1,\ldots,\dims{\inpvars}$}
            Fit surrogate on $\observeddata_{\boiteration}$\;
            Draw $\tilde{\obsouts}_{\mcsamples}\sim \probability(\obsouts;\manifoldmap_{\boiteration}(\cusvector{\inpvars}_{\boiteration,c}))$ for $\mcsamples=1,\ldots,\MCsamples$\;
            Estimate diagonal pullback Fisher entries as mentioned in~\eqref{eq:fitr_diag_estimate}\;
            Compute $\TRregularization_{\boiteration}$ using estimated diagonal entries $\widehat{\pullbackFIM}_{\boiteration,j}(\cusvector{\inpvars}_{\boiteration,c})$\;
            Form $\TRlengthweight_{\boiteration,j}\leftarrow (\widehat{\pullbackFIM}_{\boiteration,j}(\cusvector{\inpvars}_{\boiteration,c})+\TRregularization_{\boiteration})^{-1/2}$ for each $j=1,\ldots,\dims{\inpvars}$\;
            \Return $\cusvector[bm]{\TRlengthweight}_{\boiteration}= ( \TRlengthweight_{\boiteration,j} )_{j=1,\ldots,\dims{\inpvars}}$\;
        \end{algorithm}
    \end{minipage}
\end{wrapfigure}

As explained in Appendix~\ref{apx:turbo}, TuRBO centers the trust region at the best observed location $\cusvector{\inpvars}_{\mathrm{best}}$ and scales the dimension-wise side lengths using kernel lengthscales~\citep{eriksson2019scalable}. This construction is effective because it avoids performing the inner AF optimization in large regions where the surrogate has little information. Since the kernels are isotropic, the notion of distance underlying TuRBO remains Euclidean. TuRBO has proven effective in practice, but is also limited to SE or Matérn kernels.

Section~\ref{sec:manifold_view} suggests a different local notion of distance. The acquisition depends on $\cusvector{\inpvars}$ through $\manifoldmap_{\boiteration}(\cusvector{\inpvars})$, so the relevant local geometry is defined by the pullback Fisher tensor $\cusvector{\pullbackFIM}_{\boiteration}(\cusvector{\inpvars})$ in~\eqref{eq:pullback_tensor}. This tensor can be computed for any differentiable probabilistic surrogate that satisfies Assumption~\ref{ass:surrogate_regularity}. The ideal local Fisher region is the ellipsoid
\begin{equation}
    \label{eq:fitr_full_ellipsoid}
    \left\{\cusvector{\inpvars}\in\inpvarsset :
    (\cusvector{\inpvars}-\cusvector{\inpvars}_{\boiteration,c})^\top
    \cusvector{\pullbackFIM}_{\boiteration}(\cusvector{\inpvars}_{\boiteration,c})
    (\cusvector{\inpvars}-\cusvector{\inpvars}_{\boiteration,c})
    \leq \TRlength_{\boiteration}^2\right\}.
\end{equation}
% \paragraph{KL interpretation.}
% Using FITR has a direct probabilistic interpretation. For a smooth predictive family, the FIM is the second-order Taylor expansion of Kullback-Leibler (KL) divergence~\citep{amari2000methods, amari2016information} as shown in Appendix~\ref{apx:kl_trust_region}.
% As shown in Appendix~\ref{apx:kl_trust_region}, for $\delta=\cusvector{\inpvars}-\cusvector{\inpvars}_{\boiteration,c}$,
% \begin{equation}
%     \label{eq:fitr_kl_motivation}
%     D_{\mathrm{KL}}\!\left(
%     \probability(\obsouts;\manifoldmap_{\boiteration}(\cusvector{\inpvars}_{\boiteration,c}))
%     \,\middle\|\,
%     \probability(\obsouts;\manifoldmap_{\boiteration}(\cusvector{\inpvars}_{\boiteration,c}+\delta))
%     \right)
%     \approx
%     \tfrac{1}{2}\delta^\top \pullbackFIM_{\boiteration}(\cusvector{\inpvars}_{\boiteration,c}) \delta,
% \end{equation}up to second-order terms. 
\paragraph{Limitations.}
The trust region, local ellipsoid, requires inversion or whitening of the pullback tensor. For scalar predictive families, the pullback rank is at most two, so the tensor is rank-deficient when $\dims{\inpvars}>2$. Without regularization, near-null directions would produce arbitrarily large radii. A separate issue is sampling: popular BO frameworks such as BoTorch largely support sampling from axis-aligned trust regions. One can create a hyperrectangle from the pullback tensor by using the infinity norm and whitening to make it axis-aligned. However, once such a region intersects the constraints/bounds of $\inpvarsset$, the feasible set becomes a clipped polytope rather than a simple box. Polytope sampling is substantially slower and less efficient than simple axis-aligned Sobol sampling for multi-start AF optimization.

\paragraph{Proposed method.}
The experiments in this paper use a diagonal approximation of the pullback Fisher tensor rather than the full ellipsoid in~\eqref{eq:fitr_full_ellipsoid}. At the trust-region center $\cusvector{\inpvars}_{\boiteration,c}$, we first estimate the diagonal entries of the unregularized pullback Fisher tensor by Monte Carlo (MC) samples from the predictive posterior:
\begin{equation}
    \label{eq:fitr_diag_estimate}
    \widehat{\pullbackFIM}_{\boiteration,j}(\cusvector{\inpvars}_{\boiteration,c})
    =
    \frac{1}{\MCsamples}
    \sum_{\mcsamples=1}^{\MCsamples}
    \left(
    \left.\frac{\partial}{\partial \inpvars_j}
    \log \probability(\tilde{\obsouts}_{\mcsamples}\,;\manifoldmap_{\boiteration}(\cusvector{\inpvars}))\right|_{\cusvector{\inpvars}=\cusvector{\inpvars}_{\boiteration,c}}
    \right)^2,
    \qquad
    \tilde{\obsouts}_{\mcsamples}\sim
    \probability(\obsouts;\manifoldmap_{\boiteration}(\cusvector{\inpvars}_{\boiteration,c})),
\end{equation}
where $\MCsamples$ is the number of MC samples and $j \in \{1,\ldots,\dims{\inpvars}\}$. Appendix~\ref{apx:fitr_estimator} shows that~\autoref{eq:fitr_diag_estimate} is an unbiased MC estimator of the corresponding diagonal entry~\citep{martens2020new}. For a scalar Gaussian predictive, the diagonal is also available exactly from~\eqref{eq:pullback_gaussian}, without score samples. The reported BO runs use the MC estimator. Regularization is added after the estimation step.

We use diagonal Tikhonov-style regularization~\citep{martens2020new} to avoid unbounded stretching in directions associated with zero or near-zero eigenvalues of the pullback tensor:
\[
    \TRlengthweight_{\boiteration,j}=\left(\widehat{\cusvector{\pullbackFIM}}_{\boiteration,j}(\cusvector{\inpvars}_{\boiteration,c})+\TRregularization_{\boiteration}\right)^{-1/2}.
\]
% \begin{equation}
%     \label{eq:fitr_weights}
% \TRlengthweight_{\boiteration,j}=\left(\widehat{\pullbackFIM}_{\boiteration,j}(\cusvector{\inpvars}_{c})+\TRregularization_{\boiteration}\right)^{-1/2},
% \end{equation}
The regularization parameter is chosen adaptively from the local Fisher scale as $\TRregularization_{\boiteration}=\frac{1}{\dims{\inpvars}}\sum_{j=1}^{\dims{\inpvars}} \widehat{\cusvector{\pullbackFIM}}_{\boiteration,j}(\cusvector{\inpvars}_{\boiteration,c})+ \varepsilon_{\mathrm{jit}}$, where $\varepsilon_{\mathrm{jit}}$ is a small numerical jitter to prevent division by zero. This damping adds an isotropic curvature floor and interpolates between induced geometry and Euclidean geometry when the pullback tensor is degenerate~\citep{martens2020new}.
% However, it does not provide a formal guarantee of avoiding local optima exploitation behavior. 
The $\ell^2$ norm gives an ellipsoid from~\eqref{eq:fitr_full_ellipsoid}. When we use the $\ell_\infty$ norm, we obtain a box-shaped TR. We implement axis-aligned FITR using only the diagonal of the regularized pullback tensor. This yields a location-dependent axis-aligned box that remains compatible with standard BO toolchains and avoids the clipped-polytope sampling problem at the domain boundary. The diagonal approximation is $\mathrm{O}(\dims{\inpvars})$ time and storage; compared to $\mathrm{O}(\dims{\inpvars}^{2})$ for a dense whitening matrix, keeping the method practical in higher dimensions. TuRBO computes trust-region weights as in Algorithm~\ref{alg:turbo_tr_length} and then uses volume normalization and trust-region-bounded AF optimization to select the next candidate, as outlined in Algorithm~\ref{alg:turbo_candidate}. FITR replaces only the local weight construction in Algorithm~\ref{alg:turbo_tr_length} and keeps the same volume normalization and outer control logic. The resulting trust region is a box in~\eqref{eq:fitr_box}.
\begin{equation}
    \label{eq:fitr_box}
    \inpvarsset^{\mathcal{TR}}_{\boiteration}
    =
    \left\{
    \cusvector{\inpvars}\in\inpvarsset :
    \left\|\mathrm{diag}(\cusvector[bm]{\TRlengthweight}_{\boiteration})^{-1}
    (\cusvector{\inpvars}-\cusvector{\inpvars}_{\boiteration,c})\right\|_{\infty}
    \leq \frac{\TRlength_{\boiteration}}{2}
    \right\},
\end{equation}
where $\cusvector[bm]{\TRlengthweight}_{\boiteration}$ denotes the volume-normalized weights obtained from the raw values returned by Algorithm~\ref{alg:fitr_weights}. The only algorithmic change relative to TuRBO is therefore the local weight construction: ARD lengthscales are replaced by regularized pullback-Fisher sensitivities, while the trust-region scaling rule and outer control logic remain unchanged. We refer to the full whitening-based construction as FullFITR and treat it as an ablation rather than the default method.

\section{Experiments}
\label{sec:experiments}
\begin{figure}[t]
    \centering
    \includegraphics[width=0.98\linewidth]{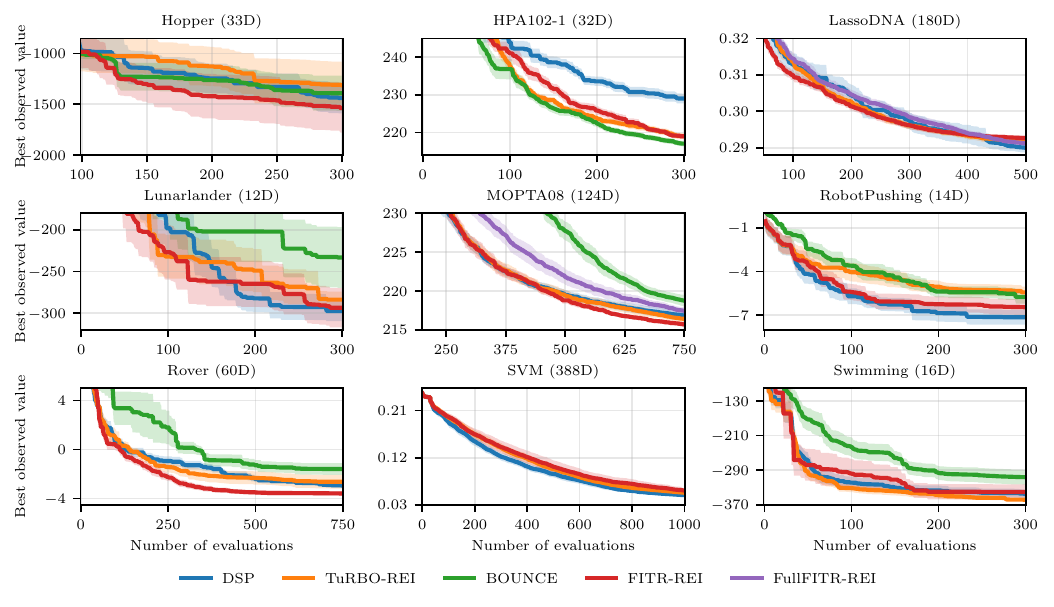}
    \caption{BO performance (lower is better) comparison for different methods with a GP using an SE kernel. Solid curves show means over $11$ trials, and shaded bands show standard error. FITR-REI and TuRBO-REI use the same GP modeling pipeline and regional trust-region center selection; their trust-region shapes use ARD lengthscales and the pullback Fisher tensor, respectively.}
    \label{fig:bo_results}
\end{figure}

\paragraph{Problems and baselines.}
We evaluate FITR with a GP using an SE kernel against strong baselines such as TuRBO~\citep{eriksson2019scalable}, DSP~\citep{hvarfner2024vanilla}, and BOUNCE~\citep{papenmeier2023bounce} on nine continuous benchmark problems, as shown in Figure~\ref{fig:bo_results}. These are control and robotics benchmarks commonly used in high-dimensional BO~\citep{eriksson2019scalable,papenmeier2023bounce}. HPA102-1 is a human-powered aircraft design task from the HPA benchmark suite~\citep{namura2025single}.
MOPTA08 is an automotive design benchmark~\citep{eriksson2019scalable}; LassoDNA and SVM test sparse regression and hyperparameter tuning behavior. Implementation details are given in Appendix~\ref{apx:implementation_details}. We also evaluate the same FITR construction with a non-lengthscale kernel, namely IBNN~\citep{li2023study}, as a transferability check beyond standard ARD kernels. Figures~\ref{fig:bo_results} and~\ref{fig:bo_results_ibnn} report the best observed objective value in the benchmark convention, while the BO procedure itself is written in maximization form.
% \footnote{Code implementation: \url{https://github.com/Sam4896/FITR-BO}}

\begin{figure}[t]
    \centering
    \includegraphics[width=0.98\linewidth]{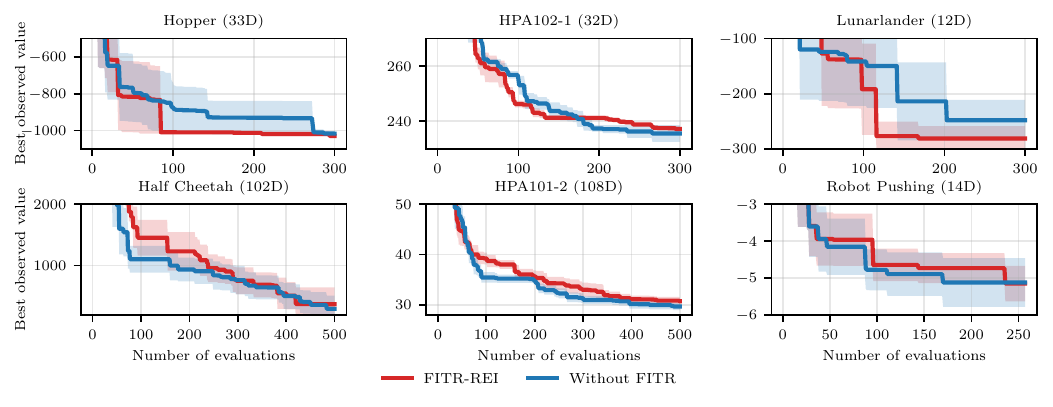}
    \caption{BO performance comparison (lower is better) with a GP using an IBNN kernel of depth $3$. Solid curves show means over $5$ trials, and shaded bands show standard error.}
    \label{fig:bo_results_ibnn}
\end{figure}

\paragraph{Results}
\label{sec:results}

Figure~\ref{fig:bo_results} shows that, for a GP with an SE kernel, FITR has lower observed mean than TuRBO-REI on several benchmarks and remains competitive with strong baselines such as DSP and BOUNCE on the remainder. Table~\ref{tab:final_observed} reports the final best observed values for different benchmarks.
\begin{table}[t]
    \centering
    \caption{Final best observed values for GP(SE), reported as mean $\pm$ standard error (lower is better). Relative improvement of FITR-REI over TuRBO-REI has also been shown (higher is better).}
    \label{tab:final_observed}
    \small
    \setlength{\tabcolsep}{4pt}
    \begin{tabular}{lrrr}
        \toprule
        Problem & TuRBO-REI & FITR-REI & Improvement (\%) \\
        \midrule
        Hopper (33D) & $-1312.41 \pm 231.02$ & $-1536.81 \pm 242.86$ & $+17.10$ \\
        HPA102-1 (32D) & $219.01 \pm 0.82$ & $218.99 \pm 0.60$ & $+0.01$ \\
        LassoDNA (180D) & $0.29195 \pm 0.00093$ & $0.29263 \pm 0.00066$ & $-0.24$ \\
        Lunarlander (12D) & $-283.66 \pm 9.69$ & $-293.36 \pm 23.62$ & $+3.42$ \\
        MOPTA08 (124D) & $216.40 \pm 0.46$ & $215.70 \pm 0.27$ & $+0.32$ \\
        RobotPushing (14D) & $-5.43 \pm 0.40$ & $-6.45 \pm 0.54$ & $+18.79$ \\
        Rover (60D) & $-2.63 \pm 0.55$ & $-3.58 \pm 0.19$ & $+36.18$ \\
        SVM (388D) & $0.05336 \pm 0.00306$ & $0.05780 \pm 0.00613$ & $-8.32$ \\
        Swimming (16D) & $-358.25 \pm 1.80$ & $-340.55 \pm 17.87$ & $-4.94$ \\
        \bottomrule
    \end{tabular}
    \vspace{-1em}
\end{table}
Figure~\ref{fig:fitr_log_diagnostics} in Appendix~\ref{apx:implementation_details} reports the corresponding trust-region diagnostics. With diagonal Tikhonov regularization, the anisotropy remains bounded while adapting to local predictive geometry. The trust-region length traces for FITR and TuRBO are qualitatively similar. Appendix~\ref{apx:runtime} reports the time to compute the Fisher diagonal used for the FITR weights.
% Using the pullback tensor, the TR length using FITR is inversely proportional to the sensitivity of the predictive distribution with respect to the input axes. This is exactly the asymmetry TuRBO encodes only through ARD lengthscales but with the added behavior of local geometry.

On the subset of problems where we evaluated FullFITR, it did not outperform diagonal FITR. A likely reason is that the pullback geometry is local. If the geometry varies rapidly near the trust-region center, the rotated region estimated at the center need not model nearby candidate locations better than the diagonal approximation. Boundary clipping and the higher numerical cost of the full tensor may also contribute. We therefore use diagonal FITR as the practical default.

The IBNN experiments in Figure~\ref{fig:bo_results_ibnn} use the same pipeline with a depth-$3$ IBNN kernel~\citep{li2023study}. In this setting we compare FITR only against standard BO with acquisition optimization over the full input domain $\inpvarsset$, so the figure should be read as a transferability check rather than as a full non-lengthscale benchmark comparison. FITR improves on the non-trust-region baseline on a subset of tasks and remains comparable on the remainder. The results show that the same FITR pipeline can be used for a non-lengthscale kernel, although the performance gain is task-dependent.
\vspace{-0.5em}
\section{Limitations}
\label{sec:limitations}
Proposition~\ref{prop:universal_bound} bounds the acquisition gradient during inner optimization; it does not imply a regret guarantee or a global convergence statement for the outer BO loop. One could also replace the trust region by a KL-divergence constraint and solve the resulting problem with an Augmented Lagrangian Method (ALM)~\citep{birgin2014practical}. During our experiments, however, the ALM variant became highly exploitative and less exploratory, which degraded performance. For non-isotropic kernels like IBNNs, the pullback trace need not collapse with $\dims{\inpvars}$ in the same way as for many lengthscale-dominated GPs~\citep{li2023study}, so the vanishing-gradient narrative from Section~\ref{sec:fim_explanations} applies less sharply. The improvement in performance using FITR becomes task-dependent. A precise analysis of when local Fisher geometry is advantageous beyond isotropic kernels is left for future work.
\vspace{-0.5em}

\section{Conclusion}
\label{sec:conclusion}

We propose an information-geometric perspective on Bayesian optimization (BO). The surrogate is treated as a smooth map from the input space to the manifold of predictive distributions, and reparameterizable acquisition functions depend on the input only through these predictive parameters. Proposition~\ref{prop:universal_bound} bounds the acquisition gradient using an acquisition-side sensitivity and the trace of the pullback Fisher tensor. For bounded acquisition-side sensitivity, the gradient vanishes when the pullback trace collapses. This provides a common explanation for several observations in high-dimensional BO, including dimension-scaled lengthscales and local initialization around observed data. Using the gradient bound, we also identify necessary conditions for an input location to be an optimizer of AF for a univariate Gaussian predictive posterior.

Using these insights, we propose the Fisher-Information Trust Region (FITR), a simple trust-region method for differentiable probabilistic surrogates that satisfy the stated regularity assumptions. FITR shapes the trust region using local pullback-Fisher weights at the current center and replaces the lengthscale-based scaling in TuRBO by a diagonal approximation to the regularized pullback Fisher tensor. In our experiments, this diagonal FITR variant improves on TuRBO-REI on several GP benchmarks and remains competitive elsewhere. FITR does not introduce task-specific tuning and can be extended to other predictive families and to batch or multi-objective settings.

This perspective opens several research directions in BO. These could include new theoretical insights, manifold-based AF optimization, or new regret bounds. The present work studies a theoretical aspect of BO; we do not foresee any broader societal impacts.

\section{Acknowledgment}
The research in this paper was supported by the Deutsche Forschungsgemeinschaft (DFG, German Research Foundation) - 459291153 (FOR5399).

\bibliographystyle{unsrtnat}
\bibliography{style/references}

\appendix
\newpage\section{Gaussian Process}\label{apx:gaussian_process}
A Gaussian process (GP) is a distribution over functions, where we place a prior on the unknown objective and update our belief using the noisy observations~\cite{rasmussen2003gaussian}.

We model the latent function as
\[
    \processfunction \sim \GP\!\big(\gpmean(\cdot), \gpkernel(\cdot,\cdot')\big),
\]
where $\gpmean(\cdot)$ is the prior mean and $\gpkernel(\cdot,\cdot')$ is the covariance kernel. For observed input-output pairs $\observeddata_{\boiteration}=\{(\cusvector{\inpvars}_{\numobservation},\obsouts_{\numobservation})\}_{\numobservation=1}^{\Numobservation}$, we use the Gaussian observation model
\[
    \obsouts_{\numobservation}=\processfunction(\cusvector{\inpvars}_{\numobservation})+\noise_{\numobservation},\qquad
    \noise_{\numobservation}\sim\mathcal{N}(0,\noisevar).
\]
The variance $\noisevar$ captures observation noise and appears as a diagonal correction to the kernel matrix.

Given $\observeddata_{\boiteration}=\{(\cusvector{\inpvars}_{\numobservation},\obsouts_{\numobservation})\}_{\numobservation=1}^{\Numobservation}$, the predictive distribution at a query point $\cusvector{\inpvars}$ is Gaussian
\[
    \probability(\obsouts\mid \cusvector{\inpvars},\observeddata_{\boiteration})
    =\mathcal{N}\!\left(\distmean_{\boiteration}(\cusvector{\inpvars}),\,\distcov_{\boiteration}(\cusvector{\inpvars})\right),
\]
where the posterior mean and variance are given by standard Gaussian conditioning identities~\cite{rasmussen2003gaussian}
\begin{align}
    \distmean_{\boiteration}(\cusvector{\inpvars})
     & =
    \gpmean(\cusvector{\inpvars})
    +\gpkernel(\cusvector{\inpvars},\cusmatrix{\Inpvars})
    \left(\gpkernel(\cusmatrix{\Inpvars},\cusmatrix{\Inpvars})+\noisevar\mathbf{I}\right)^{-1}
    \left(\cusvector{\obsouts}-\gpmean(\cusmatrix{\Inpvars})\right), \\
    \distcov_{\boiteration}(\cusvector{\inpvars})
     & =
    \gpkernel(\cusvector{\inpvars},\cusvector{\inpvars})
    -\gpkernel(\cusvector{\inpvars},\cusmatrix{\Inpvars})
    \left(\gpkernel(\cusmatrix{\Inpvars},\cusmatrix{\Inpvars})+\noisevar\mathbf{I}\right)^{-1}
    \gpkernel(\cusmatrix{\Inpvars},\cusvector{\inpvars}).
\end{align}
and the GP hyperparameters (for example, lengthscales in $\GPhyperparam$ and the noise level $\noisevar$) are estimated by maximizing the log marginal likelihood~\cite{rasmussen2003gaussian}
\begin{equation}
    \log \probability(\cusvector{\obsouts}\mid \cusmatrix{\Inpvars},\GPhyperparam)
    =
    -\frac{1}{2}\left(\cusvector{\obsouts}-\gpmean(\cusmatrix{\Inpvars})\right)^{\top}
    \GPkernel_{\GPhyperparam}^{-1}
    \left(\cusvector{\obsouts}-\gpmean(\cusmatrix{\Inpvars})\right)
    -\frac{1}{2}\log\!\left|\GPkernel_{\GPhyperparam}\right|
    -\frac{\Numobservation}{2}\log(2\pi).
\end{equation}
where $\GPkernel_{\GPhyperparam}=\gpkernel(\cusmatrix{\Inpvars},\cusmatrix{\Inpvars})+\noisevar\mathbf{I}$~\cite{rasmussen2003gaussian}.

\section{Common Acquisition Functions: Analytic and MC Forms}\label{apx:acqf_table}
Section~\ref{sec:preliminaries} introduces AF and its reparameterized form~\eqref{eq:acqf_expectation}. When the acquisition is available in closed form, $\acqfdistfunctional$ does not depend on $\cusvector[bm]{\MCbasesamples}$ and~\eqref{eq:acqf_expectation} collapses to
\begin{equation}
    \label{eq:acqf_analytic}
    \acqf(\cusvector{\inpvars}) = \mathbb{E}_{\cusvector[bm]{\MCbasesamples}}\!\left[\acqfdistfunctional\!\left(\cusvector[bm]{\manifoldparameters}_{\boiteration,\cusvector{\inpvars}};\cusvector[bm]{\acqfparams}\right)\right]
    = \acqfdistfunctional\!\left(\cusvector[bm]{\manifoldparameters}_{\boiteration,\cusvector{\inpvars}};\cusvector[bm]{\acqfparams}\right).
\end{equation}

Table~\ref{tab:acqf_analytic} shows some closed-form AF for a scalar Gaussian predictive, where $z=(\distmean_{\boiteration,\cusvector{\inpvars}}-\obsouts^{\incumbentobs})/\distcov_{\boiteration,\cusvector{\inpvars}}^{1/2}$, $\probabilitycdf_{\scriptscriptstyle\mathcal{N}}$ is the Cumulative Distribution Function (CDF) and $\probabilitypdf_{\scriptscriptstyle\mathcal{N}}$ is Probability Density Function (PDF) of normal distribution.

\begin{table}[hb]
    \centering
    \caption{Closed-form AF $\acqfdistfunctional(\cusvector[bm]{\manifoldparameters}_{\boiteration,\cusvector{\inpvars}};\cusvector[bm]{\acqfparams})$ under~\eqref{eq:acqf_analytic} for a scalar Gaussian predictive $\cusvector[bm]{\manifoldparameters}_{\boiteration,\cusvector{\inpvars}}=(\distmean_{\boiteration,\cusvector{\inpvars}},\distcov_{\boiteration,\cusvector{\inpvars}})$. EI and PI use incumbent $\obsouts^{\incumbentobs}$; UCB uses exploration weight $\ucbexploration>0$.}
    \label{tab:acqf_analytic}
    \vspace{0.5em}
    \small
    \begingroup
    \renewcommand{\tabularxcolumn}[1]{>{\centering\arraybackslash}m{#1}}
    \setlength{\tabcolsep}{2pt}
    \renewcommand{\arraystretch}{1.45}
    \begin{tabularx}{0.7\linewidth}{@{} >{\centering\arraybackslash}m{1.75cm} X @{}}
        \toprule
        \textbf{Acquisition} & \textbf{Closed-form $\acqfdistfunctional(\distmean_{\boiteration,\cusvector{\inpvars}},\distcov_{\boiteration,\cusvector{\inpvars}};\cusvector[bm]{\acqfparams})$}                                                                                \\
        \midrule\addlinespace[0.15em]
        EI                   & $\bigl(\distmean_{\boiteration,\cusvector{\inpvars}}-\obsouts^{\incumbentobs}\bigr)\,\probabilitycdf_{\scriptscriptstyle\mathcal{N}}(z) + \distcov_{\boiteration,\cusvector{\inpvars}}^{1/2}\,\probabilitypdf_{\scriptscriptstyle\mathcal{N}}(z)$ \\
        \addlinespace[0.35em]
        PI                   & $\probabilitycdf_{\scriptscriptstyle\mathcal{N}}(z)$                                                                                                                                                                                              \\
        \addlinespace[0.35em]
        UCB                  & $\distmean_{\boiteration,\cusvector{\inpvars}} + \sqrt{\ucbexploration}\,\distcov_{\boiteration,\cusvector{\inpvars}}^{1/2}$                                                                                                                \\
        \bottomrule
    \end{tabularx}
    \endgroup
\end{table}

Following~\citet{wilson2017reparameterization}, the Monte Carlo form of EI can be written as
\begin{equation}
    \label{eq:ei_mc}
    \acqfdistfunctional_{\mathrm{EI}}(\cusvector[bm]{\manifoldparameters}_{\boiteration, \cusvector{\inpvars}},\,
    \cusvector[bm]{\MCbasesamples};\,\cusvector[bm]{\acqfparams})
    = \mathrm{ReLU}\!\left(\distmean_{\boiteration, \cusvector{\inpvars}}
    + \distcov_{\boiteration, \cusvector{\inpvars}}^{1/2}\cusvector[bm]{\MCbasesamples}
    - \obsouts^{\incumbentobs}\right).
\end{equation}

For parallel BO~\citep{wang2020parallel} at batch input $\batchinpvars=\{\cusvector{\inpvars}_1, \ldots, \cusvector{\inpvars}_{\batchsize}\}$, we write the reparameterized sample as $\cusvector[bm]{\obsouts}=\cusvector[bm]{\distmean}_{\boiteration,\batchinpvars}+\cusvector[bm]{\choleskyfactor}_{\boiteration,\batchinpvars}\cusvector[bm]{\MCbasesamples}$, with $\cusvector[bm]{\choleskyfactor}_{\boiteration,\batchinpvars}=\mathrm{chol}(\cusvector{\distcov}_{\boiteration,\batchinpvars})$ and $\cusvector[bm]{\MCbasesamples}\sim\mathcal{N}(0,\bm{I})$. Table~\ref{tab:acqf} uses this batch notation because it matches the BoTorch implementation. In the scalar sequential setting of the main text, $\batchsize=1$ and these expressions reduce to the usual one-point acquisitions.

For MC-UCB in the scalar case, the reparameterized form uses the identity $\mathbb{E}|Z|=\sqrt{2/\pi}$ for $Z\sim\mathcal{N}(0,1)$, so $\sqrt{\ucbexploration\pi/2}\,\distcov_{\boiteration,\cusvector{\inpvars}}^{1/2}|Z|$ has expectation $\sqrt{\ucbexploration}\,\distcov_{\boiteration,\cusvector{\inpvars}}^{1/2}$.

\begin{table}[ht]
    \caption{Utility $\acqfutility$ and integrand $\acqfdistfunctional(\cdot,\cusvector[bm]{\MCbasesamples};\cdot)$ inside $\mathbb{E}_{\cusvector[bm]{\MCbasesamples}}$ for reparameterized acquisitions under~\eqref{eq:acqf_expectation} for a Gaussian predictive~\citep{wilson2017reparameterization,wang2020parallel}. $\mathbf{1}^{+/-}$ denote the right/left-continuous Heaviside step functions; $\mathrm{ReLU}(u)=\max(u,0)$; $\mathrm{sigmoid}(u)=(1+e^{-u})^{-1}$; $\mathrm{H}$ is Shannon entropy. EI (Expected Improvement) and PI (Probability of Improvement) use incumbent threshold $\obsouts^{\incumbentobs}$; PI uses temperature $\pitemperature>0$; the PI row is a sigmoid-smoothed surrogate for Monte Carlo gradients and is not the analytic PI functional $\probabilitycdf_{\scriptscriptstyle\mathcal{N}}(z)$. UCB uses exploration weight $\ucbexploration>0$. For ES (Entropy Search), subscripts $a$ and $b$ mark the query set and discretization points; $\cusvector[bm]{\distmean}_{b\mid a}$ and $\cusvector[bm]{\choleskyfactor}_{b\mid a}$ are the conditional posterior mean and Cholesky factor at $b$ given an outcome at $a$. SR (Simple Regret) does not use extra parameters.}
    \vspace{0.5em}
    \label{tab:acqf}
    \centering
    \small
    \begingroup
    \renewcommand{\tabularxcolumn}[1]{>{\centering\arraybackslash}m{#1}}
    \setlength{\tabcolsep}{5pt}
    \renewcommand{\arraystretch}{1.45}
    \begin{tabularx}{\linewidth}{@{} >{\centering\arraybackslash}m{1.75cm} X X @{}}
        \toprule
        \textbf{Acquisition} & \textbf{Utility $\acqfutility(\obsouts; \cusvector[bm]{\manifoldparameters}_{\boiteration, \cusvector{\inpvars}},\cusvector[bm]{\acqfparams})$}                                                            & \textbf{Reparameterized form $\acqfdistfunctional(\cusvector[bm]{\manifoldparameters}_{\boiteration, \cusvector{\inpvars}}, \cusvector[bm]{\MCbasesamples}; \cusvector[bm]{\acqfparams})$}                                                        \\
        \midrule\addlinespace[0.15em]
        EI                   & $\max\!\left(\mathrm{ReLU}\!\left(\cusvector[bm]{\obsouts}-\obsouts^{\incumbentobs}\right)\right)$                                                                                                         & $\max\!\left(\mathrm{ReLU}\!\left(\cusvector[bm]{\distmean}_{\boiteration,\batchinpvars}+\cusvector[bm]{\choleskyfactor}_{\boiteration,\batchinpvars}\cusvector[bm]{\MCbasesamples}-\obsouts^{\incumbentobs}\right)\right)$                       \\
        \addlinespace[0.35em]
        PI                   & $\max\!\left(\mathbf{1}^{-}\!\left(\cusvector[bm]{\obsouts}-\obsouts^{\incumbentobs}\right)\right)$                                                                                                        & $\max\!\left(\mathrm{sigmoid}\!\left(\frac{\cusvector[bm]{\distmean}_{\boiteration,\batchinpvars}+\cusvector[bm]{\choleskyfactor}_{\boiteration,\batchinpvars}\cusvector[bm]{\MCbasesamples}-\obsouts^{\incumbentobs}}{\pitemperature}\right)\right)$       \\
        \addlinespace[0.35em]
        UCB                  & $\max\!\left(\cusvector[bm]{\distmean}_{\boiteration,\batchinpvars}+\sqrt{\ucbexploration\pi/2}\,\bigl|\cusvector[bm]{\obsouts}-\cusvector[bm]{\distmean}_{\boiteration,\batchinpvars}\bigr|\right)$ & $\max\!\left(\cusvector[bm]{\distmean}_{\boiteration,\batchinpvars}+\sqrt{\ucbexploration\pi/2}\,\bigl|\cusvector[bm]{\choleskyfactor}_{\boiteration,\batchinpvars}\cusvector[bm]{\MCbasesamples}\bigr|\right)$                             \\
        \addlinespace[0.35em]
        SR                   & $\max(\cusvector[bm]{\obsouts})$                                                                                                                                                                           & $\max\!\left(\cusvector[bm]{\distmean}_{\boiteration,\batchinpvars}+\cusvector[bm]{\choleskyfactor}_{\boiteration,\batchinpvars}\cusvector[bm]{\MCbasesamples}\right)$                                                                            \\
        \addlinespace[0.35em]
        ES                   & $-\mathrm{H}\!\left(\mathbb{E}_{\cusvector[bm]{\obsouts}_b\mid\cusvector[bm]{\obsouts}_a}\!\left[\mathbf{1}^{+}\!\left(\cusvector[bm]{\obsouts}_b-\max(\cusvector[bm]{\obsouts}_b)\right)\right]\right)$   & $-\mathrm{H}\!\left(\mathbb{E}_{\cusvector[bm]{\MCbasesamples}_b}\!\left[\mathrm{softmax}\!\left(\frac{\cusvector[bm]{\distmean}_{b\mid a}+\cusvector[bm]{\choleskyfactor}_{b\mid a}\cusvector[bm]{\MCbasesamples}_b}{\tau}\right)\right]\right)$ \\
        \bottomrule
    \end{tabularx}
    \endgroup
\end{table}

\subsection{Trust-Region Bayesian Optimization (TuRBO)}
\label{apx:turbo}
TuRBO replaces a single global search step with local BO updates inside the trust region~\citep{eriksson2019scalable}. At each iteration $\boiteration$, a trust region (TR), $\inpvarsset^{\mathcal{TR}}_{\boiteration}\subseteq\inpvarsset$, is defined by its center $\cusvector{\inpvars}_{\boiteration, c}$ and length $\TRlength_{\boiteration}$. The TR is an axis-aligned box for which the dimension-wise length is calculated using the GP surrogate $\surrogatemodel_{\boiteration}$ with Automatic Relevance Determination (ARD) kernel. The next candidate is selected by maximizing an AF with bounds as
\[
    \cusvector{\inpvars}_{\boiteration+1}
    =
    \arg\max_{\cusvector{\inpvars}\in\inpvarsset^{\mathcal{TR}}_{\boiteration}}
    \acqf(\cusvector{\inpvars};\observeddata_{\boiteration}).
\]
The TR length $\TRlength_{\boiteration}$ grows or shrinks using two counters. $\TRsuccesscount$ tracks the number of consecutive improvements over the incumbent $\obsouts^{\incumbentobs}$, and $\TRfailurecount$ tracks the number of consecutive failures. If $\TRsuccesscount$ reaches the threshold $\TRsuccesscount$, the region expands; if $\TRfailurecount$ reaches $\TRfailurecount$, it contracts; both counters reset on any length change as
\[
    \TRlength_{\boiteration+1}
    =
    \begin{cases}
        \min(2\TRlength_{\boiteration}, \TRlength_{\max}), & \text{if success count reaches }\TRsuccesscount, \\
        \tfrac{1}{2}\TRlength_{\boiteration},              & \text{if failure count reaches }\TRfailurecount, \\
        \TRlength_{\boiteration},                          & \text{otherwise.}
    \end{cases}
\]

\begin{algorithm}[t]
    \caption{TuRBO TR weights calculation}
    \label{alg:turbo_tr_length}
    \KwIn{Center $\cusvector{\inpvars}_{\boiteration, c}$, data $\observeddata_{\boiteration}$}
    \KwOut{TR length weights $\cusvector[bm]{\TRlengthweight}_{\boiteration}$, trained surrogate GP $\surrogatemodel_{\boiteration}$}
    Fit GP on $\observeddata_{\boiteration}$ and extract ARD lengthscales $\lengthscale_{\boiteration,j}$ for all $j=1,\ldots,\dims{\inpvars}$\;
    $\TRlengthweight_{\boiteration,j} \leftarrow \lengthscale_{\boiteration,j}$ \quad for each $j=1,\ldots,\dims{\inpvars}$\;
    \Return $\cusvector[bm]{\TRlengthweight}_{\boiteration}= ( \TRlengthweight_{\boiteration,j} )_{j=1,\ldots,\dims{\inpvars}}$ and $\surrogatemodel_{\boiteration}$\;
\end{algorithm}

\begin{algorithm}[t]
    \caption{Trust-region length control and next candidate selection}
    \label{alg:turbo_candidate}
    \KwIn{TR-weight routine \texttt{compute\_tr\_weights}, budget $\Boiteration$, thresholds $\TRsuccesscount^{\mathrm{max}},\TRfailurecount^{\mathrm{max}}$, bounds $\TRlength_{\min},\TRlength_{\max}$ and initial length $\TRlength_{\mathrm{init}}$}
    \KwOut{Best input $\cusvector{\inpvars}^{\incumbentobs}$ and value $\obsouts^{\incumbentobs}$}
    Initialize $\cusvector{\inpvars}_{0,c}$, $\TRlength_{0}\leftarrow\TRlength_{\mathrm{init}}$, $\TRsuccesscount\leftarrow 0$, $\TRfailurecount\leftarrow 0$, $\observeddata_{0}$, incumbent $\obsouts^{\incumbentobs}$\;
    \For{$\boiteration=0,\ldots,\Boiteration-1$}{
    Get TR length weights $\TRlengthweight_j$ and $\surrogatemodel_{\boiteration}$ from \texttt{compute\_tr\_weights} with $(\cusvector{\inpvars}_{\boiteration, c},\observeddata_{\boiteration})$\;
    \textbf{Volume normalization}\;
    $\TRlengthweight_j \leftarrow \TRlengthweight_j \big/\bigl(\prod_{k=1}^{\dims{\inpvars}}\TRlengthweight_k\bigr)^{1/\dims{\inpvars}}$ \quad for each $j=1,\ldots,\dims{\inpvars}$\;
    \textbf{TR construction using upper and lower limits of $\inpvarsset$}\;
    \For{$j=1,\ldots,\dims{\inpvars}$}{
        $\alpha_j \leftarrow \max\!\Bigl(\inpvars_j^c - \tfrac{\TRlengthweight_j\TRlength_{\boiteration}}{2},\;\inpvars_j^{\min}\Bigr)$\;
        $\beta_j  \leftarrow \min\!\Bigl(\inpvars_j^c + \tfrac{\TRlengthweight_j\TRlength_{\boiteration}}{2},\;\inpvars_j^{\max}\Bigr)$\;
    }
    Define $\inpvarsset^{\mathcal{TR}}_{\boiteration}$ by the dimension-wise limits $[\alpha_j,\,\beta_j]$ for all $j=1,\ldots,\dims{\inpvars}$\;
    Select next candidate $\cusvector{\inpvars}_{\boiteration+1}\leftarrow\arg\max_{\cusvector{\inpvars}\in\inpvarsset^{\mathcal{TR}}_{\boiteration}}\acqf(\cusvector{\inpvars};\observeddata_{\boiteration})$\;
    Evaluate $\obsouts_{\boiteration+1}=\processfunction_{\mathrm{true}}(\cusvector{\inpvars}_{\boiteration+1})$ and update $\observeddata_{\boiteration+1}$\;
    \eIf{$\obsouts_{\boiteration+1}>\obsouts^{\incumbentobs}$}{
        $\TRsuccesscount\leftarrow\TRsuccesscount+1$; $\TRfailurecount\leftarrow 0$\;
        Update $\obsouts^{\incumbentobs}$\;
    }{
        $\TRsuccesscount\leftarrow 0$; $\TRfailurecount\leftarrow\TRfailurecount+1$\;
    }
    $\TRlength_{\boiteration+1}\leftarrow\begin{cases}
            \min(2\TRlength_{\boiteration},\TRlength_{\max}), & \TRsuccesscount=\TRsuccesscount^{\mathrm{max}} \\
            \tfrac{1}{2}\TRlength_{\boiteration},             & \TRfailurecount=\TRfailurecount^{\mathrm{max}} \\
            \TRlength_{\boiteration},                         & \text{otherwise}
        \end{cases}$\;
    Reset $\TRsuccesscount,\TRfailurecount\leftarrow 0$ if $\TRlength$ changed\;
    \If{$\TRlength_{\boiteration+1}<\TRlength_{\min}$}{
        Restart around a new center; $\TRlength_{\boiteration+1}\leftarrow\TRlength_{\mathrm{init}}$\;
    }
    }
    \Return $\cusvector{\inpvars}_{\boiteration+1}^{\incumbentobs}$ and $\obsouts^{\incumbentobs}$\;
\end{algorithm}

When $\TRlength_{\boiteration}<\TRlength_{\min}$ at any iteration, TR is said to have collapsed. It is then restarted around a new center location. \citet{namura2025regional} provide a Regional Expected Improvement (REI) acquisition function that can be used to select the next TR center location when the TR collapses. \citet{eriksson2019scalable} also present a multi-region variant with multiple trust-region centers; however, we omit it here.

\newpage\section{Implementation Details}
\label{apx:implementation_details}
For the experiments with a GP using an SE kernel, we use the Dimensionally Scaled LogNormal Priors (DSP) with mean $\sqrt{2} + \log(\dims{\inpvars}) * 0.5$ and standard deviation $\sqrt{3}$. DSP~\citep{hvarfner2024vanilla} uses dimensionally scaled LogNormal priors on the lengthscales and optimizes AF over the full input domain $\inpvarsset$, which is also the default implementation for SingleTaskGP in BoTorch v0.16.1~\citep{balandat2020botorch}. For the experiments with the GP using an IBNN kernel, we keep the BoTorch default depth $3$.

TuRBO-REI combines TuRBO's lengthscale-weighted trust region~\citep{eriksson2019scalable} with Regional Expected Improvement for choosing the trust-region center when the TR collapses~\citep{namura2025regional}. \citet{namura2025regional} showed that REI improved TuRBO by selecting a new center in a region with high expected utility. BOUNCE~\citep{papenmeier2023bounce} uses an adaptive subspace decomposition and can optimize mixed and combinatorial structures as well. FITR keeps the same trust-region scaling logic as TuRBO-REI (Algorithm~\ref{alg:turbo_candidate}) and also uses REI center selection at trust-region collapse. It replaces ARD lengthscale weights by the regularized diagonal pullback-Fisher weights from Algorithm~\ref{alg:fitr_weights}. Section~\ref{sec:fitr} makes clear that the main method in this paper is the diagonal axis-aligned approximation; we also test a full-tensor ablation (FullFITR). FullFITR uses whitening for Sobol sampling and a simple boundary-clipping strategy instead of polytope MCMC. Due to its additional computational cost, we evaluate FullFITR only on a subset of problems.

DSP, TuRBO-REI, FITR, and FullFITR are implemented in BoTorch~\citep{balandat2020botorch} and use LogEI~\citep{ament2023unexpected}. Inputs are normalized to $[0,1]^{\dims{\inpvars}}$, and outputs are standardized before fitting the GP. For each experiment, we start from $30$ randomly sampled observed points. We use sequential evaluations ($\batchsize=1$). The SE-kernel experiments are repeated over $11$ independently seeded trials, whereas the IBNN experiments are repeated over $5$ independently seeded trials. AF is optimized using L-BFGS with $5$ restarts and $256$ raw samples. For FITR, the diagonal entries of the pullback Fisher tensor $\pullbackFIM_{\boiteration}(\cusvector{\inpvars}_{\boiteration,c})$ are estimated with $\MCsamples=256$ MC samples using~\eqref{eq:fitr_diag_estimate}. These settings are fixed across tasks rather than tuned per benchmark.

The damping parameter is chosen automatically as $\TRregularization_{\boiteration}=\dims{\inpvars}^{-1}\sum_{j=1}^{\dims{\inpvars}}\widehat{\pullbackFIM}_{\boiteration,j}+10^{-12}$. The resulting weights are volume-normalized before constructing the trust region. Using Algorithm~\ref{alg:turbo_candidate}, the trust-region length $\TRlength_{\boiteration}$ controls the search volume while the pullback tensor controls the local anisotropy. The trust-region length follows the standard TuRBO rule: it doubles after $10$ consecutive improvements, halves after $\min\{\lceil\max(\frac{4}{\batchsize},\frac{\dims{\inpvars}}{\batchsize})\rceil,20\}$ consecutive failures, and restarts when the length drops below $0.5^7$. FullFITR uses the full regularized pullback Fisher tensor and optimizes LogEI in whitened coordinates. Candidate points are mapped back to the input cube and clipped at the boundary, which keeps the implementation compatible with standard BoTorch box optimization while exposing the additional cost of the full tensor. The experiments were run on an Nvidia GPU Tesla V100-SXM2-32GB.

\subsection{Runtime for computing the Fisher diagonal}
\label{apx:runtime}

Computing the Fisher diagonal used to form the FITR weights took $242$--$261$ ms with $256$ MC samples, compared to $2.5$--$4.1$ ms using the exact Gaussian expression in Corollary~\ref{cor:gaussian}. These isolated CPU timings use synthetic fixed-kernel posteriors with $30$ observations for $32$, $124$, and $388$ inputs dimensions.

\subsection{Geometry illustrations}
\label{apx:theory_illustration}

Figure~\ref{fig:theoretical_insights}(a) uses an SE GP with eight observations in $[0,1]^2$, lengthscales $(0.22,0.24)$, and observation-noise variance $0.0025$. We compare the EI gradient at $(0.70,0.30)$ with a stationary point found numerically.

For Figure~\ref{fig:theoretical_insights}(b,c), we use a zero-mean GP with SE kernel with and observation-noise variance $0.1$. In (c), $\lengthscale=0.35\dims{\inpvars}^{\lengthscaleexponent}$ with $\lengthscaleexponent\in\{1/4,1/2,3/4\}$, using the same uniform candidates for all three powers.

\begin{figure}[htbp]
    \centering
    \begin{adjustbox}{max width=0.68\textwidth,max height=0.42\textheight,center}
        \includegraphics{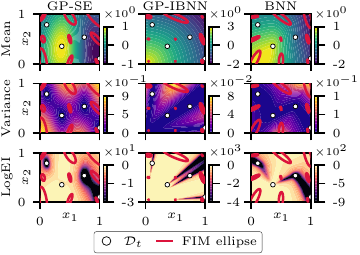}
    \end{adjustbox}
    \caption{Ellipses using $\pullbackFIM(\cusvector{\inpvars})$ from~\eqref{eq:pullback_tensor} for different surrogates: GP with squared exponential (SE) kernel, GP with Infinite Width Bayesian Neural Network (IBNN) kernel~\cite{li2023study} and Bayesian Neural Network (BNN).}
    \label{fig:fitr_ellipses}
\end{figure}

\subsection{Trust-region diagnostics}
\label{sec:fitr_log_diagnostics}

\begin{figure*}[h]
    \centering
    \includegraphics[width=0.6\textwidth]{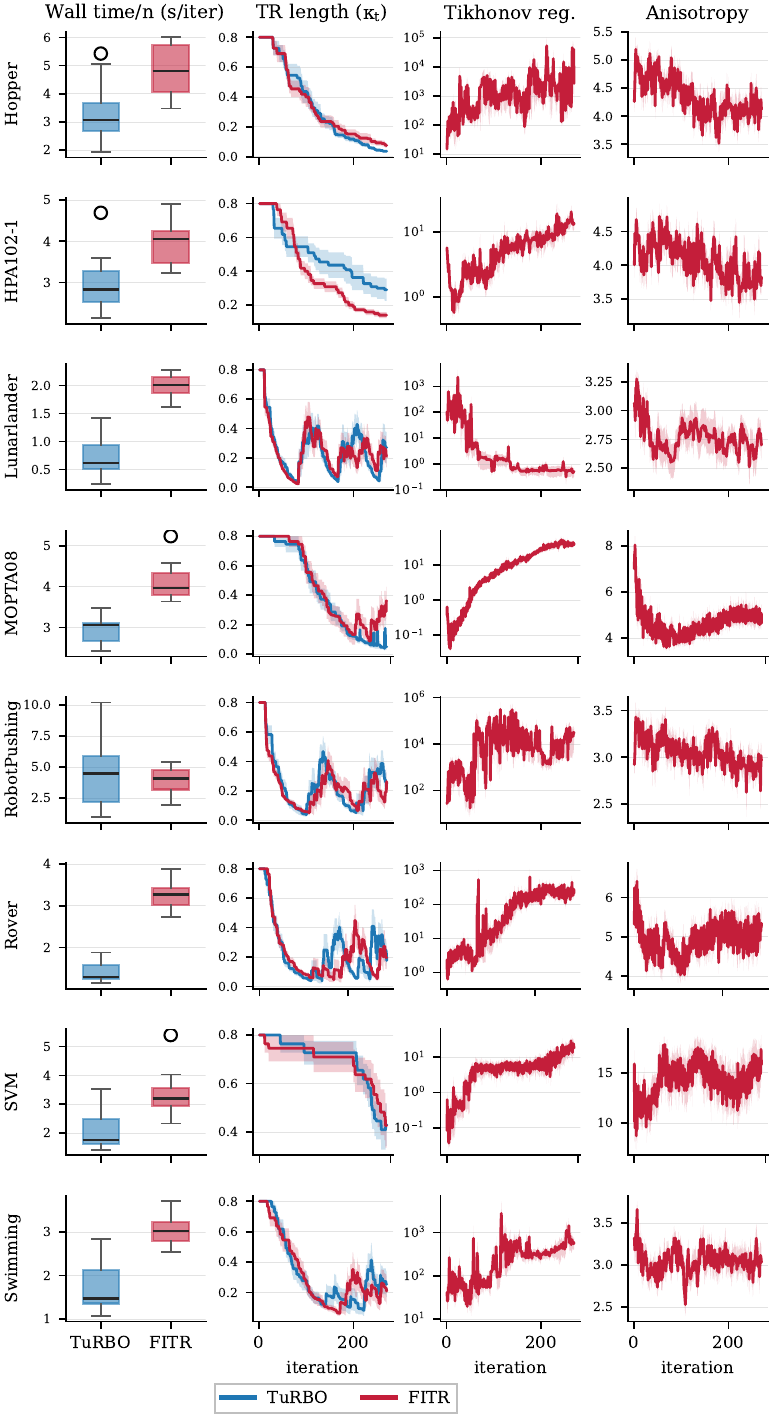}
    \caption{Per-problem diagnostics for TuRBO and FITR on the GP-SE benchmark suite. The first column reports the average time for getting the next candidate in BO iterations. The second column shows the trust-region length $\kappa_t$. The last two columns show the FITR-specific Tikhonov regularization and anisotropy diagnostics. Solid lines show means over the $11$ seeded runs and shaded bands show standard error.}
    \label{fig:fitr_log_diagnostics}
\end{figure*}

Figure~\ref{fig:fitr_log_diagnostics} complements the objective-value plots with internal trust-region diagnostics. The trust-region length traces for FITR and TuRBO-REI remain qualitatively similar across problems, so the performance differences in Figure~\ref{fig:bo_results} are not explained by a different restart or length-adaptation schedule. For the diagonal FITR implementation, the logged anisotropy is the ratio between the largest and smallest volume-normalized coordinate weights used to shape the axis-aligned trust region, so values close to $1$ indicate a nearly isotropic local region and larger values indicate stronger directional stretching. The diagnostics are shown using the experiment logs whose results are reported in Figure~\ref{fig:bo_results}.

\clearpage\section{Proofs}
\label{apx:proofs}

\subsection{Interchange Conditions for Gaussian Reparameterizations}
\label{apx:acqf_regularity}

Differentiating under the expectation in Proposition~\ref{prop:universal_bound} requires a valid interchange of gradient and expectation. We use standard reparameterization conditions from~\citet{wilson2018maximizing}. The lemma below provides a sufficient condition for scalar Gaussian reparameterizations, where the AF integrand depends on the sampled output through the predictive distribution parameters $\cusvector[bm]{\manifoldparameters}$. This is the case used later for MC-EI and the smoothed MC-PI surrogate.
% Other reparameterized acquisitions, such as MC-UCB, need the same pathwise differentiability argument but do not fit the exact template below.

\begin{lemma}[Generic Gaussian Reparameterization Condition]
    \label{lem:acqf_regularity}
    Let Assumption~\ref{ass:surrogate_regularity} hold and let
    \[
        \acqfdistfunctional\!\left(\cusvector[bm]{\manifoldparameters}_{\boiteration,\cusvector{\inpvars}},
        \cusvector[bm]{\MCbasesamples};\cusvector[bm]{\acqfparams}\right)
        = \acqfutility\!\left(\distmean_{\boiteration,\cusvector{\inpvars}}
        + \distcov_{\boiteration,\cusvector{\inpvars}}^{1/2}\cusvector[bm]{\MCbasesamples};\,
        \cusvector[bm]{\acqfparams}\right)
    \]
    be a scalar Gaussian reparameterized acquisition,
    with $\cusvector[bm]{\MCbasesamples} \sim \mathcal{N}(0, I)$ independent of $\cusvector{\inpvars}$~\citep{wilson2018maximizing,wang2020parallel}.
    We assume
    \begin{itemize}
        \item[\textup{(a)}] $\acqfdistfunctional(\cdot, \cusvector[bm]{\MCbasesamples}; \cusvector[bm]{\acqfparams})$ is locally Lipschitz on $\manifold$ and differentiable almost everywhere for almost every $\cusvector[bm]{\MCbasesamples}$;
        \item[\textup{(b)}] $\acqfutility(\obsouts;\cusvector[bm]{\acqfparams})$ is globally Lipschitz in $\obsouts$.
    \end{itemize}
    Then
    $\mathbb{E}_{\cusvector[bm]{\MCbasesamples}}\!\left[\left\|
        \nabla_{\cusvector[bm]{\manifoldparameters}_{\boiteration,\cusvector{\inpvars}}}
        \acqfdistfunctional\!\left(\cusvector[bm]{\manifoldparameters}_{\boiteration,\cusvector{\inpvars}},
        \cusvector[bm]{\MCbasesamples};\cusvector[bm]{\acqfparams}\right)
        \right\|^2\right] < \infty$
    for all $\cusvector[bm]{\manifoldparameters}_{\boiteration,\cusvector{\inpvars}} \in \manifold$.
\end{lemma}

\begin{proof}
    Assumption~\textup{(a)} is the pathwise differentiability condition
    used in reparameterized acquisition optimization: the integrand admits gradients almost everywhere, which is sufficient for piecewise-smooth utilities such as MC-EI. For square-integrability, let $L_a$ be the
    Lipschitz constant of $\acqfutility$ from assumption~\textup{(b)}.
    The chain rule gives, for almost every $\cusvector[bm]{\MCbasesamples}$,
    \begin{equation}
        \label{eq:reparameterization_gradient}
        \nabla_{\cusvector[bm]{\manifoldparameters}_{\boiteration,\cusvector{\inpvars}}}
        \acqfdistfunctional
        = \nabla_{\obsouts}\acqfutility\!\left(
        \distmean_{\boiteration,\cusvector{\inpvars}}
        + \distcov_{\boiteration,\cusvector{\inpvars}}^{1/2}\cusvector[bm]{\MCbasesamples};\,
        \cusvector[bm]{\acqfparams}\right)
        \cdot
        \nabla_{\cusvector[bm]{\manifoldparameters}_{\boiteration,\cusvector{\inpvars}}}
        \!\left(\distmean_{\boiteration,\cusvector{\inpvars}}
        + \distcov_{\boiteration,\cusvector{\inpvars}}^{1/2}\cusvector[bm]{\MCbasesamples}\right).
    \end{equation}
    Since $\acqfutility$ is $L_a$-Lipschitz in $\obsouts$, the first factor satisfies
    $\|\nabla_{\obsouts}\acqfutility\| \le L_a$.
    The second factor is affine in $\cusvector[bm]{\MCbasesamples}$, so
    $\|\nabla_{\cusvector[bm]{\manifoldparameters}_{\boiteration,\cusvector{\inpvars}}}(\distmean_{\boiteration,\cusvector{\inpvars}} + \distcov_{\boiteration,\cusvector{\inpvars}}^{1/2}\cusvector[bm]{\MCbasesamples})\|
        \le A + B\|\cusvector[bm]{\MCbasesamples}\|$
    for finite constants $A, B$ at a fixed $\cusvector[bm]{\manifoldparameters}_{\boiteration,\cusvector{\inpvars}}$. $\mathrm{chol}(\cdot)$ is smooth on the positive-definite cone, since $\distcov_{\boiteration,\cusvector{\inpvars}}\succ 0$ by Assumption~\ref{ass:surrogate_regularity}(ii)~\citep{smith1995differentiation}.
    Therefore
    $\|\nabla_{\cusvector[bm]{\manifoldparameters}_{\boiteration,\cusvector{\inpvars}}}\acqfdistfunctional\|^2
        \le L_a^2(A + B\|\cusvector[bm]{\MCbasesamples}\|)^2$,
    and since $\cusvector[bm]{\MCbasesamples}\sim\mathcal{N}(0,I)$ has finite second moment the
    expectation is finite.
\end{proof}

\begin{remark}[Necessity of $\cusvector{\distcov} \succ 0$]
    \label{rem:variance_singularity}
    % The condition $\cusvector{\distcov}_{\boiteration} \succ 0$ keeps the Cholesky map smooth. 
    If the predictive variance degenerates, derivatives of the square-root factor can blow up, and the bound above would not hold. In the scalar case, $\nabla \distcov^{1/2} = (\nabla \distcov)/(2\distcov^{1/2})$ is singular as $\distcov \to 0$. A strictly positive noise level rules out this singularity for GP regression. Similar reasoning applies to the batch BO setting; however, we limit our discussion to the scalar case for simplicity in this work.
\end{remark}

\begin{remark}[LogEI in the experimental regime]
    \label{rem:logei_regularity}
    The experiments use the numerically stabilized LogEI implementation from BoTorch~\citep{ament2023unexpected}. In the scalar Gaussian setting with strictly positive predictive variance, this implementation is smooth in the predictive parameters on any compact subset of $\mathbb{R}\times(0,\infty)$. The local interchange argument above therefore applies in the positive-variance regime considered in the experiments. Appendix~\ref{apx:gaussian_sensitivity} distinguishes the analytic EI and LogEI sensitivities from those of Monte Carlo representations.
\end{remark}

\subsection{Proof of Proposition~\ref{prop:universal_bound} (Acquisition Gradient Bound)}
\label{apx:proof_universal}

Let $\cusvector[bm]{\MCbasesamples} \sim \mathcal{N}(0,\eye)$ be the independent base samples in the reparameterized acquisition from Lemma~\ref{lem:acqf_regularity}.
$\acqf(\cusvector{\inpvars}) = \mathbb{E}_{\cusvector[bm]{\MCbasesamples}}[\acqfdistfunctional(\cusvector[bm]{\manifoldparameters}_{\boiteration,\cusvector{\inpvars}}, \cusvector[bm]{\MCbasesamples}; \cusvector[bm]{\acqfparams})]$,
where $\cusvector[bm]{\manifoldparameters}_{\boiteration,\cusvector{\inpvars}} = (\distmean_{\boiteration, \cusvector{\inpvars}},\distcov_{\boiteration,\cusvector{\inpvars}})$
are the GP posterior mean and variance at $\cusvector{\inpvars}$.
$\jacobianmap_{\boiteration}(\cusvector{\inpvars}) = \partial\cusvector[bm]{\manifoldparameters}_{\boiteration,\cusvector{\inpvars}}/\partial\cusvector{\inpvars}
    \in \mathbb{R}^{\dims{\manifoldparameters} \times \dims{\inpvars}}$ is the Jacobian of the posterior map.

\par\smallskip\noindent\textbf{Step 1 (Chain rule and interchange).}
By the assumed interchange of gradient and expectation in Proposition~\ref{prop:universal_bound}, justified for scalar Gaussian reparameterizations by Lemma~\ref{lem:acqf_regularity}, the reparameterized acquisition satisfies
\begin{equation}
    \label{eq:chain_rule}
    \nabla_{\cusvector{\inpvars}}\,\acqf(\cusvector{\inpvars})
    = \jacobianmap_{\boiteration}(\cusvector{\inpvars})^\top\,
    \mathbb{E}_{\cusvector[bm]{\MCbasesamples}}\!\left[\nabla_{\cusvector[bm]
    {\manifoldparameters}_{\boiteration,\cusvector{\inpvars}}}\acqfdistfunctional\!\left(\manifoldmap_
        {\boiteration}(\cusvector{\inpvars}), \cusvector[bm]{\MCbasesamples}; \cusvector[bm]{\acqfparams}\right)
    \right],
\end{equation}
which also aligns with the discussion around Equation 5 of Section 3.2 in~\citet{wilson2018maximizing}. For ease of notation, we write $\mathbb{E}_{\cusvector[bm]{\MCbasesamples}}\!\left[\nabla_{\cusvector[bm]
    {\manifoldparameters}_{\boiteration,\cusvector{\inpvars}}}\acqfdistfunctional\!\left(\manifoldmap_
        {\boiteration}(\cusvector{\inpvars}), \cusvector[bm]{\MCbasesamples}; \cusvector[bm]{\acqfparams}\right) \right]$ as $\mathbb{E}_{\cusvector[bm]{\MCbasesamples}}\!\left[\nabla_{\cusvector[bm]{\manifoldparameters}}\acqfdistfunctional\!\left(\cdot,\,\cusvector[bm]{\MCbasesamples}\right)\right]$ in this proof, where the argument $\cdot$ denotes $\manifoldmap_{\boiteration}(\cusvector{\inpvars})$.

\par\smallskip\noindent\textbf{Step 2 (Fisher metric insertion).}
Using Assumption~\ref{ass:surrogate_regularity}(ii), let $\cusvector{\pullbackFIM}_{\manifold}(\cusvector{\manifoldparameters}_{\boiteration,\cusvector{\inpvars}}) := \fishermetric_{\manifold}(\manifoldmap_{\boiteration}(\cusvector{\inpvars})) \succ 0$ . Using the identity $\eye = \cusvector{\pullbackFIM}_{\manifold}(\cusvector{\manifoldparameters}_{\boiteration,\cusvector{\inpvars}})^{1/2}\cusvector{\pullbackFIM}_{\manifold}(\cusvector{\manifoldparameters}_{\boiteration,\cusvector{\inpvars}})^{-1/2}$ from~\citet{amari2000methods}, we can decompose the gradient into a geometry-dependent factor $\cusmatrix{\fisherjacobian}$ and an acquisition-dependent factor $\cusvector{\fisheracqgradient}$:
\begin{equation}
    \label{eq:fisher_metric_insertion}
    \nabla_{\cusvector{\inpvars}}\,\acqf(\cusvector{\inpvars})
    = \underbrace{\left(\cusvector{\pullbackFIM}_{\manifold}(\cusvector{\manifoldparameters}_{\boiteration,\cusvector{\inpvars}})^{1/2}\jacobianmap_{\boiteration}(\cusvector{\inpvars})\right)^\top}_{=:\,\cusmatrix{\fisherjacobian}^\top}\, \underbrace{\cusvector{\pullbackFIM}_{\manifold}(\cusvector{\manifoldparameters}_{\boiteration,\cusvector{\inpvars}})^{-1/2}\mathbb{E}_{\cusvector[bm]{\MCbasesamples}}\!\left[\nabla_{\cusvector[bm]{\manifoldparameters}}\acqfdistfunctional\!\left(\cdot,\,\cusvector[bm]{\MCbasesamples}\right)\right]}_{=:\,\cusvector{\fisheracqgradient}},
\end{equation}
\par\smallskip\noindent\textbf{Step 3 (Cauchy--Schwarz).}
For any matrix $\cusmatrix{\fisherjacobian} \in \mathbb{R}^{p \times q}$ and vector $\cusvector{\fisheracqgradient} \in \mathbb{R}^p$,
$\|\cusmatrix{\fisherjacobian}^\top \cusvector{\fisheracqgradient}\|_2^2 \leq \|\cusmatrix{\fisherjacobian}\|_F^2 \|\cusvector{\fisheracqgradient}\|^2$. Applying it to the decomposition from Step~2 gives
\begin{equation}
    \label{eq:cauchy_schwarz}
    \left\|\nabla_{\cusvector{\inpvars}}\,\acqf(\cusvector{\inpvars})\right\|_2\leq \sqrt{\underbrace{\left\|\cusvector{\pullbackFIM}_{\manifold}(\cusvector{\manifoldparameters}_{\boiteration,\cusvector{\inpvars}})^{1/2}\jacobianmap_{\boiteration}(\cusvector{\inpvars})\right\|_F^2}_{=\,\mathrm{tr}(\cusvector{\pullbackFIM}_{\boiteration}(\cusvector{\inpvars}))}
    \cdot \left\|\cusvector{\pullbackFIM}_{\manifold}(\cusvector{\manifoldparameters}_{\boiteration,\cusvector{\inpvars}})^{-1/2}\mathbb{E}_{\cusvector[bm]{\MCbasesamples}}\!\left[\nabla_{\cusvector[bm]{\manifoldparameters}}\acqfdistfunctional\!\left(\cdot,\,\cusvector[bm]{\MCbasesamples}\right)\right]\right\|^2},
\end{equation}
where using~\eqref{eq:pullback_tensor} we have $\mathrm{tr}\!\left(\cusvector{\pullbackFIM}_{\boiteration}(\cusvector{\inpvars})\right) = \mathrm{tr}\!\left(\jacobianmap_{\boiteration}^\top\,\fishermetric_{\manifold}\!\left(\manifoldmap_{\boiteration}(\cusvector{\inpvars})\right)\jacobianmap_{\boiteration}\right) = \left\|\fishermetric_{\manifold}^{1/2}\!\left(\manifoldmap_{\boiteration}(\cusvector{\inpvars})\right)\jacobianmap_{\boiteration}\right\|_F^2$. And \begin{equation}
    \label{eq:jensen_inequality_step_3}
    \left\|\cusvector{\pullbackFIM}_{\manifold}(\cusvector{\manifoldparameters}_{\boiteration,\cusvector{\inpvars}})^{-1/2}\mathbb{E}_{\cusvector[bm]{\MCbasesamples}}\!\left[\nabla_{\cusvector[bm]{\manifoldparameters}}\acqfdistfunctional\!\left(\cdot,\,\cusvector[bm]{\MCbasesamples}\right)\right]\right\|^2=\mathbb{E}_{\cusvector[bm]{\MCbasesamples}}\!\left[\nabla_{\cusvector[bm]{\manifoldparameters}}\acqfdistfunctional\!\left(\cdot,\,\cusvector[bm]{\MCbasesamples}\right)\right]^\top\, \cusvector{\pullbackFIM}_{\manifold}(\cusvector{\manifoldparameters}_{\boiteration,\cusvector{\inpvars}})^{-1}\, \mathbb{E}_{\cusvector[bm]{\MCbasesamples}}\!\left[\nabla_{\cusvector[bm]{\manifoldparameters}}\acqfdistfunctional\!\left(\cdot,\,\cusvector[bm]{\MCbasesamples}\right)\right].
\end{equation}

\par\smallskip\noindent\textbf{Step 4 (Jensen's inequality).}
Considering  $v = \mathbb{E}_{\cusvector[bm]{\MCbasesamples}}\!\left[\nabla_{\cusvector[bm]{\manifoldparameters}}\acqfdistfunctional\!\left(\cdot,\,\cusvector[bm]{\MCbasesamples}\right)\right]$,~\eqref{eq:jensen_inequality_step_3} can be rewritten as $f: v \mapsto v^\top \cusvector{\pullbackFIM}_{\manifold}(\cusvector{\manifoldparameters}_{\boiteration,\cusvector{\inpvars}})^{-1} v$, where $f$ is convex for $\cusvector{\pullbackFIM}_{\manifold}(\cusvector{\manifoldparameters}_{\boiteration,\cusvector{\inpvars}}) \succ 0$. Using Jensen's inequality we can write $f(\mathbb{E}_{\cusvector[bm]{\MCbasesamples}}\!\left[\nabla_{\cusvector[bm]{\manifoldparameters}}\acqfdistfunctional\!\left(\cdot,\,\cusvector[bm]{\MCbasesamples}\right)\right]) \leq \mathbb{E}_{\cusvector[bm]{\MCbasesamples}}\!\left[f(\nabla_{\cusvector[bm]{\manifoldparameters}}\acqfdistfunctional\!\left(\cdot,\,\cusvector[bm]{\MCbasesamples}\right))\right]$. Hence, we have
\begin{equation}
    \label{eq:jensen_inequality}
    \left\|\cusvector{\pullbackFIM}_{\manifold}(\cusvector{\manifoldparameters}_{\boiteration,\cusvector{\inpvars}})^{-1/2}\mathbb{E}_{\cusvector[bm]{\MCbasesamples}}\!\left[\nabla_{\cusvector[bm]{\manifoldparameters}}\acqfdistfunctional\!\left(\cdot,\,\cusvector[bm]{\MCbasesamples}\right)\right]\right\|^2
    \leq
    \mathbb{E}_{\cusvector[bm]{\MCbasesamples}}\!\left[
        \nabla_{\cusvector[bm]
            {\manifoldparameters}}\acqfdistfunctional\!\left(\cdot,\,\cusvector[bm]{\MCbasesamples}\right)^\top\,
        \cusvector{\pullbackFIM}_{\manifold}(\cusvector{\manifoldparameters}_{\boiteration,\cusvector{\inpvars}})^{-1}\,
        \nabla_{\cusvector[bm]
            {\manifoldparameters}}\acqfdistfunctional\!\left(\cdot,\,\cusvector[bm]{\MCbasesamples}\right)
        \right]
    = \acqfsensitivity\!\left(\manifoldmap_{\boiteration}(\cusvector{\inpvars})\right).
\end{equation}

Combining~\eqref{eq:cauchy_schwarz} and~\eqref{eq:jensen_inequality} gives~\eqref{eq:universal_bound}. However, it is to noted that for analytical AFs, this step (Jensen's inequality) is an equality since that integrand is independent of the random base samples $\cusvector[bm]{\MCbasesamples}$. Any remaining slack is bound comes from the Cauchy--Schwarz step. Equality holds if and only if $\cusvector{\fisheracqgradient}=0$ or every column of $\cusvector{\fisherjacobian}$ is a scalar multiple of $\cusvector{\fisheracqgradient}$.

% \begin{remark}[Connection to classical information-geometric bounds]
%     The factorization in~\eqref{eq:fisher_metric_insertion} (inserting $\cusvector{\pullbackFIM}_{\manifold}(\cusvector{\manifoldparameters}_{\boiteration,\cusvector{\inpvars}})^{1/2}\cusvector{\pullbackFIM}_{\manifold}(\cusvector{\manifoldparameters}_{\boiteration,\cusvector{\inpvars}})^{-1/2}$ and applying Cauchy--Schwarz)
%     is the algebraic core of the Cram\'{e}r--Rao lower bound~\citep{amari2000methods}
%     and the Fisher SAM perturbation~\citep{kim2022fisher}.
%     In the Cram\'{e}r--Rao setting, the same identity lower-bounds estimator variance by the
%     inverse Fisher information; in Fisher SAM, it identifies $\cusvector{\pullbackFIM}_{\manifold}(\cusvector{\manifoldparameters}_{\boiteration,\cusvector{\inpvars}})^{-1}\nabla\ell$
%     as the worst-case loss direction under a Fisher-metric constraint.
%     The present bound is the dual: the Fisher matrix upper-bounds the acquisition gradient
%     by separating the contribution of the input-to-parameter geometry from the
%     acquisition-specific sensitivity $\acqfsensitivity$.
% \end{remark}

\subsection{Gaussian Predictive Distribution}
\label{apx:proof_gaussian}

\begin{corollary}[Pullback FIM for a Gaussian Predictive Distribution]
    \label{cor:gaussian}
    Let Assumption~\ref{ass:surrogate_regularity} hold and suppose the predictive distribution is univariate Gaussian,
    $\probability(\obsouts \vert \cusvector{\inpvars}, \observeddata_{\boiteration})
        = \mathcal{N}\!\bigl(\distmean_{\boiteration, \cusvector{\inpvars}},
        \distcov_{\boiteration,\cusvector{\inpvars}}\bigr)$,
    where $\distmean_{\boiteration, \cusvector{\inpvars}} \in \mathbb{R}$ is the posterior
    mean and $\distcov_{\boiteration,\cusvector{\inpvars}} > 0$ is the posterior variance,
    parameterized by
    $\cusvector[bm]{\manifoldparameters}_{\boiteration,\cusvector{\inpvars}}
        = \bigl(\distmean_{\boiteration, \cusvector{\inpvars}},\distcov_{\boiteration,\cusvector{\inpvars}}\bigr) \in \manifold = \mathbb{R} \times \mathbb{R}_{+}$.
    The Fisher--Rao metric at this point on the predictive distribution manifold is
    \begin{equation}
        \label{eq:fim_gaussian}
        \fishermetric_{\manifold}\!\bigl(\cusvector[bm]{\manifoldparameters}_{\boiteration,\cusvector{\inpvars}}\bigr)
        = \mathrm{diag}\!\left(
        \distcov_{\boiteration,\cusvector{\inpvars}}^{-1},\;
        \tfrac{1}{2}\,\distcov_{\boiteration,\cusvector{\inpvars}}^{-2}
        \right),
    \end{equation}
    and the pullback tensor~\eqref{eq:pullback_tensor} reduces to
    \begin{equation}
        \label{eq:pullback_gaussian}
        \cusvector{\pullbackFIM}_{\boiteration}(\cusvector{\inpvars})
        = \frac{\nabla\distmean_{\boiteration, \cusvector{\inpvars}}\,
            \nabla\distmean_{\boiteration, \cusvector{\inpvars}}^\top}
        {\distcov_{\boiteration,\cusvector{\inpvars}}}
        + \frac{\nabla\distcov_{\boiteration,\cusvector{\inpvars}}\,
            \nabla\distcov_{\boiteration,\cusvector{\inpvars}}^\top}
        {2\,\distcov_{\boiteration,\cusvector{\inpvars}}^{2}},
    \end{equation}
    where $\nabla$ denotes the gradient with respect to $\cusvector{\inpvars}$.
\end{corollary}
The first term, $\nabla\distmean_{\boiteration, \cusvector{\inpvars}}\nabla\distmean_{\boiteration, \cusvector{\inpvars}}^\top /\distcov_{\boiteration, \cusvector{\inpvars}}$, is large where the posterior mean changes rapidly relative to uncertainty, so it captures the local exploitation geometry. The second term, $\nabla\distcov_{\boiteration, \cusvector{\inpvars}}\nabla\distcov_{\boiteration, \cusvector{\inpvars}}^\top /(2\distcov_{\boiteration, \cusvector{\inpvars}}^{2})$, is large where the posterior variance changes rapidly, so it captures the local exploration geometry.
% The Gaussian pullback tensor therefore combines mean and variance sensitivity in a single positive semidefinite rank-at-most-two tensor on $\inpvarsset$.

\begin{proof}[Proof of Corollary~\ref{cor:gaussian}]
    The log-likelihood of $\mathcal{N}(\obsouts;\distmean_{\boiteration, \cusvector{\inpvars}},\distcov_{\boiteration, \cusvector{\inpvars}})$ is
    $\log\probability = -\tfrac{1}{2}\log(2\pi\distcov_{\boiteration, \cusvector{\inpvars}}) - (\obsouts-\distmean_{\boiteration, \cusvector{\inpvars}})^2/(2\distcov_{\boiteration, \cusvector{\inpvars}})$.
    Differentiating with respect to the distribution parameters $\cusvector[bm]{\manifoldparameters} = (\distmean_{\boiteration, \cusvector{\inpvars}}, \distcov_{\boiteration, \cusvector{\inpvars}})$ gives the score functions
    \[
        \frac{\partial \log \probability}{\partial \distmean_{\boiteration, \cusvector{\inpvars}}} = \frac{\obsouts - \distmean_{\boiteration, \cusvector{\inpvars}}}{\distcov_{\boiteration, \cusvector{\inpvars}}},
        \qquad
        \frac{\partial \log \probability}{\partial \distcov_{\boiteration, \cusvector{\inpvars}}} = \frac{(\obsouts - \distmean_{\boiteration, \cusvector{\inpvars}})^2}{2\distcov_{\boiteration, \cusvector{\inpvars}}^2} - \frac{1}{2\distcov_{\boiteration, \cusvector{\inpvars}}}.
    \]
    The off-diagonal terms would be zero~\cite{amari2016information}, and the FIM is a diagonal matrix
    % \[
    %     \mathbb{E}\!\left[
    %         \frac{\partial \log\probability}{\partial \distmean_{\boiteration, \cusvector{\inpvars}}}
    %         \frac{\partial \log\probability}{\partial \distcov_{\boiteration, \cusvector{\inpvars}}}
    %         \right]
    %     = \frac{\mathbb{E}[(\obsouts-\distmean_{\boiteration, \cusvector{\inpvars}})^3]}{2\distcov_{\boiteration, \cusvector{\inpvars}}^3}
    %     - \frac{\mathbb{E}[\obsouts-\distmean_{\boiteration, \cusvector{\inpvars}}]}{2\distcov_{\boiteration, \cusvector{\inpvars}}^2} = 0,
    % \]
    % since all odd central moments of a Gaussian are zero.
    % The diagonal entries follow from
    % $\mathbb{E}[(\obsouts-\distmean_{\boiteration, \cusvector{\inpvars}})^2] = \distcov_{\boiteration, \cusvector{\inpvars}}$ and
    % $\mathbb{E}[(\obsouts-\distmean_{\boiteration, \cusvector{\inpvars}})^4] = 3\distcov_{\boiteration, \cusvector{\inpvars}}^2$:
    \[
        \fishermetric_{\manifold}\!\bigl(\cusvector[bm]{\manifoldparameters}_{\boiteration,\cusvector{\inpvars}}\bigr)
        = \mathrm{diag}\!\left(\distcov_{\boiteration, \cusvector{\inpvars}}^{-1},\; \tfrac{1}{2}\distcov_{\boiteration, \cusvector{\inpvars}}^{-2}\right).
    \]
    The posterior map $\manifoldmap_{\boiteration}$ has Jacobian
    $\jacobianmap_{\boiteration}(\cusvector{\inpvars})
        = \bigl[\nabla\distmean_{\boiteration, \cusvector{\inpvars}},\;\nabla\distcov_{\boiteration,\cusvector{\inpvars}}\bigr]^\top
        \in \mathbb{R}^{2 \times \dims{\inpvars}}$,
    where each row is the gradient of one parameter with respect to $\cusvector{\inpvars}$.
    Substituting into~\eqref{eq:pullback_tensor}:
    \begin{align*}
        \cusvector{\pullbackFIM}_{\boiteration}(\cusvector{\inpvars})
         & = \jacobianmap_{\boiteration}^\top\,
        \mathrm{diag}\!\bigl(\distcov_{\boiteration, \cusvector{\inpvars}}^{-1},\;
        \tfrac{1}{2}\distcov_{\boiteration, \cusvector{\inpvars}}^{-2}\bigr)\,
        \jacobianmap_{\boiteration}                               \\
         & = \distcov_{\boiteration, \cusvector{\inpvars}}^{-1}\,
        \nabla\distmean_{\boiteration, \cusvector{\inpvars}}\,\nabla\distmean_{\boiteration, \cusvector{\inpvars}}^\top
        + \tfrac{1}{2}\distcov_{\boiteration, \cusvector{\inpvars}}^{-2}\,
        \nabla\distcov_{\boiteration, \cusvector{\inpvars}}\,\nabla\distcov_{\boiteration, \cusvector{\inpvars}}^\top.
    \end{align*}
\end{proof}

\subsection{Acquisition Sensitivity under a Scalar Gaussian Predictive}
\label{apx:gaussian_sensitivity}

Substituting $\fishermetric_{\manifold}^{-1} = \mathrm{diag}(\distcov_{\boiteration, \cusvector{\inpvars}},\, 2\distcov_{\boiteration, \cusvector{\inpvars}}^{2})$ from~\eqref{eq:fim_gaussian} into the definition of $\acqfsensitivity$ in~\eqref{eq:jensen_inequality} gives
\begin{equation}
    \label{eq:sensitivity_gaussian_generic}
    \acqfsensitivity\!\bigl(\cusvector[bm]{\manifoldparameters}_{\boiteration}\bigr)
    = \mathbb{E}_{\cusvector[bm]{\MCbasesamples}}\!\left[
        \distcov_{\boiteration, \cusvector{\inpvars}}\!\left(\frac{\partial\acqfdistfunctional}{\partial\distmean_{\boiteration, \cusvector{\inpvars}}}\right)^{\!2}
        + 2\distcov_{\boiteration, \cusvector{\inpvars}}^{2}\!\left(\frac{\partial\acqfdistfunctional}{\partial\distcov_{\boiteration, \cusvector{\inpvars}}}\right)^{\!2}
        \right].
\end{equation}
For both analytic and Monte Carlo AFs, boundedness of $\acqfsensitivity$ reduces to controlling partial derivatives of $\acqfdistfunctional$ with respect to $(\distmean_{\boiteration, \cusvector{\inpvars}}, \distcov_{\boiteration, \cusvector{\inpvars}})$.

\paragraph{Analytic acquisition functions.}
Let $z_{\boiteration} = (\distmean_{\boiteration, \cusvector{\inpvars}} - \obsouts^{\incumbentobs})/\sqrt{\distcov_{\boiteration, \cusvector{\inpvars}}}$, $\probabilitycdf$ and $\probabilitypdf$ be the cumulative distribution function (CDF) and probability density function (PDF) of the standard normal distribution.

\paragraph{Analytic UCB.}
For Upper Confidence Bound (UCB)~\citep{srinivas2012information,garnett_bayesoptbook_2023}, $\acqfdistfunctional = \distmean_{\boiteration, \cusvector{\inpvars}} + \sqrt{\ucbexploration\distcov_{\boiteration, \cusvector{\inpvars}}}$ gives

$\partial\acqfdistfunctional/\partial\distmean_{\boiteration, \cusvector{\inpvars}} = 1$ and $\partial\acqfdistfunctional/\partial\distcov_{\boiteration, \cusvector{\inpvars}} = \sqrt{\ucbexploration}/(2\sqrt{\distcov_{\boiteration, \cusvector{\inpvars}}})$, so
\begin{equation}
    \label{eq:sensitivity_ucb}
    \acqfsensitivity^{\text{UCB}} = \distcov_{\boiteration, \cusvector{\inpvars}}\!\left(1 + \tfrac{\ucbexploration}{2}\right).
\end{equation}
\paragraph{Analytic EI.}
For Expected Improvement (EI)~\citep{mockus1998application,jones1998efficient}, $\acqfdistfunctional^{\text{EI}} = (\distmean_{\boiteration, \cusvector{\inpvars}} - \obsouts^{\incumbentobs})\probabilitycdf(z_{\boiteration}) + \sqrt{\distcov_{\boiteration, \cusvector{\inpvars}}}\,\probabilitypdf(z_{\boiteration})$ gives

$\partial z_{\boiteration}/\partial\distcov_{\boiteration, \cusvector{\inpvars}} = -z_{\boiteration}/(2\distcov_{\boiteration, \cusvector{\inpvars}})$, $\partial\acqfdistfunctional^{\text{EI}}/\partial\distmean_{\boiteration, \cusvector{\inpvars}} = \probabilitycdf(z_{\boiteration})$ and $\partial\acqfdistfunctional^{\text{EI}}/\partial\distcov_{\boiteration, \cusvector{\inpvars}} = \probabilitypdf(z_{\boiteration})/(2\sqrt{\distcov_{\boiteration, \cusvector{\inpvars}}})$. Using $\probabilitycdf \leq 1$ and $\probabilitypdf \leq 1/\sqrt{2\pi}$.
\begin{equation}
    \label{eq:sensitivity_ei}
    \acqfsensitivity^{\text{EI}} = \distcov_{\boiteration, \cusvector{\inpvars}}\!\left(\probabilitycdf(z_{\boiteration})^2 + \tfrac{1}{2}\probabilitypdf(z_{\boiteration})^2\right) \leq \distcov_{\boiteration, \cusvector{\inpvars}}\!\left(1 + \tfrac{1}{4\pi}\right),
\end{equation}

\paragraph{Analytic PI.}
For Probability of Improvement (PI)~\citep{mockus1998application}, $\acqfdistfunctional^{\text{PI}} = \probabilitycdf(z_{\boiteration})$ gives

$\partial\acqfdistfunctional^{\text{PI}}/\partial\distmean_{\boiteration, \cusvector{\inpvars}} = \probabilitypdf(z_{\boiteration})/\sqrt{\distcov_{\boiteration, \cusvector{\inpvars}}}$ and $\partial\acqfdistfunctional^{\text{PI}}/\partial\distcov_{\boiteration, \cusvector{\inpvars}} = -z_{\boiteration}\probabilitypdf(z_{\boiteration})/(2\distcov_{\boiteration, \cusvector{\inpvars}})$. Since $e^{-z^{2}}(1 + z^{2}/2) \leq 1$ for all $z \in \mathbb{R}$,
\begin{equation}
    \label{eq:sensitivity_pi}
    \acqfsensitivity^{\text{PI}} = \probabilitypdf(z_{\boiteration})^2\!\left(1 + \tfrac{z_{\boiteration}^2}{2}\right) \leq \tfrac{1}{2\pi},
\end{equation}
% The UCB and EI bounds scale linearly with $\distcov_{\boiteration, \cusvector{\inpvars}}$, while the PI bound holds uniformly on $\manifold$ without any variance condition. 
For bounded $\distcov_{\boiteration, \cusvector{\inpvars}}$, the AF sensitivities in~\eqref{eq:sensitivity_ucb},~\eqref{eq:sensitivity_ei}, and~\eqref{eq:sensitivity_pi} are bounded.

\paragraph{Analytic LogEI.}
For analytic LogEI~\citep{ament2023unexpected}, applying the chain rule to the logarithm of EI gives
\begin{equation}
    \label{eq:sensitivity_logei}
    \acqfsensitivity^{\mathrm{LogEI}}(\cusvector[bm]{\manifoldparameters})
    =
    \frac{\acqfsensitivity^{\mathrm{EI}}(\cusvector[bm]{\manifoldparameters})}
    {\acqfdistfunctional^{\mathrm{EI}}(\cusvector[bm]{\manifoldparameters})^2}.
\end{equation}

\paragraph{Monte Carlo sensitivity bounds.}
The MC column of Table~\ref{tab:sensitivity_bounds} follows from bounds on the samplewise parameter derivatives. The next lemma isolates the shared calculation for MC-EI and smoothed MC-PI.

\begin{lemma}[Monte Carlo sensitivity bound]
    \label{lem:mc_sensitivity_bound}
    Let Assumption~\ref{ass:surrogate_regularity} hold. At a fixed scalar Gaussian predictive state $\cusvector[bm]{\manifoldparameters}=(\distmean,\distcov)$ with $\distcov>0$, let $\MCbasesamples$ be a scalar standard-normal base sample. Suppose there is a finite scalar constant $\acqflipschitzbound>0$, independent of $\MCbasesamples$, such that almost surely
    \begin{equation}
        \label{eq:mc_lipschitz_condition}
        \left|\frac{\partial\acqfdistfunctional(\cusvector[bm]{\manifoldparameters},\MCbasesamples)}{\partial\distmean}\right|
        \leq \acqflipschitzbound,
        \qquad
        \left|\frac{\partial\acqfdistfunctional(\cusvector[bm]{\manifoldparameters},\MCbasesamples)}{\partial\distcov}\right|
        \leq \frac{\acqflipschitzbound\,|\MCbasesamples|}{2\sqrt{\distcov}}.
    \end{equation}
    Then
    \begin{equation}
        \label{eq:sensitivity_mc_generic}
        \acqfsensitivity(\cusvector[bm]{\manifoldparameters})
        \leq \tfrac32\acqflipschitzbound^2\distcov.
    \end{equation}
\end{lemma}

\begin{proof}
    Substituting the two derivative bounds into~\eqref{eq:sensitivity_gaussian_generic} gives
    \[
        \acqfsensitivity(\cusvector[bm]{\manifoldparameters})
        \leq \acqflipschitzbound^2\distcov
        +\tfrac12\acqflipschitzbound^2\distcov\,
        \expectation_{\MCbasesamples}[\MCbasesamples^2]
        =\tfrac32\acqflipschitzbound^2\distcov,
    \]
    since a scalar standard-normal variable has second moment one.
\end{proof}

\paragraph{MC-EI.}
The ReLU utility in Table~\ref{tab:acqf} has derivative bounded by one almost everywhere~\citep{wilson2017reparameterization,wilson2018maximizing}. Its Gaussian reparameterization satisfies both conditions with $\acqflipschitzbound=1$, giving $\acqfsensitivity^{\mathrm{MC\text{-}EI}}(\cusvector[bm]{\manifoldparameters})\leq3\distcov/2$.

\paragraph{Smoothed MC-PI.}
The sigmoid utility with temperature $\pitemperature$ in Table~\ref{tab:acqf} has derivative bounded by $1/(4\pitemperature)$~\citep{wilson2018maximizing}. Taking $\acqflipschitzbound=1/(4\pitemperature)$ gives $\acqfsensitivity^{\mathrm{MC\text{-}PI}}(\cusvector[bm]{\manifoldparameters})\leq3\distcov/(32\pitemperature^2)$. This is the sensitivity of the smoothed utility, not analytic PI.

\paragraph{MC-UCB.}
For the representation in Table~\ref{tab:acqf}, direct differentiation gives
\[
    \frac{\partial\acqfdistfunctional^{\mathrm{UCB}}(\cusvector[bm]{\manifoldparameters},\MCbasesamples)}{\partial\distmean}=1,
    \qquad
    \frac{\partial\acqfdistfunctional^{\mathrm{UCB}}(\cusvector[bm]{\manifoldparameters},\MCbasesamples)}{\partial\distcov}
    =\frac{\sqrt{\ucbexploration\pi/2}\,|\MCbasesamples|}{2\sqrt{\distcov}}.
\]
The lemma applies with $\acqflipschitzbound=\max\{1,\sqrt{\ucbexploration\pi/2}\}$. However, substituting these derivatives directly into~\eqref{eq:sensitivity_gaussian_generic} gives the sharper identity $\acqfsensitivity^{\mathrm{MC\text{-}UCB}}(\cusvector[bm]{\manifoldparameters})=\distcov(1+\ucbexploration\pi/4)$.

\begin{remark}[Extension to multivariate outputs and parallel BO]
    \label{rem:pullback_multivariate}
    For a multivariate Gaussian predictive distribution $\mathcal{N}(\cusvector[bm]{\distmean}_{\boiteration},\,\cusvector{\distcov}_{\boiteration})$ with $\cusvector[bm]{\distmean}_{\boiteration} \in \mathbb{R}^{\multiobjective}$ and $\cusvector{\distcov}_{\boiteration} \in \mathbb{R}^{\multiobjective\times\multiobjective}$ (as in Multi-Objective Bayesian Optimization (MOBO) with $\multiobjective$ correlated outputs), the Fisher--Rao metric takes the block-diagonal form $\fishermetric_{\manifold} = \mathrm{diag}\!\bigl(\cusvector{\distcov}_{\boiteration}^{-1},\,\tfrac{1}{2}\cusvector{\distcov}_{\boiteration}^{-1}\otimes\cusvector{\distcov}_{\boiteration}^{-1}\bigr)$~\citep{amari2016information}, and~\eqref{eq:pullback_gaussian} extends to
    \begin{equation}
    \label{eq:pullback_multivariate}
    \cusvector{\pullbackFIM}_{\boiteration}(\batchinpvars)    = \jacobianmap_{\cusvector[bm]{\distmean}}^\top\,        \cusvector{\distcov}_{\boiteration}^{-1}\,\jacobianmap_{\cusvector[bm]{\distmean}}        + \tfrac{1}{2}\,\jacobianmap_{\cusvector{\distcov}}^\top\!        \left(\cusvector{\distcov}_{\boiteration}^{-1}        \otimes\cusvector{\distcov}_{\boiteration}^{-1}\right)        \jacobianmap_{\cusvector{\distcov}},
    \end{equation}
    where $\jacobianmap_{\cusvector[bm]{\distmean}}$ and $\jacobianmap_{\cusvector{\distcov}}$ are the Jacobians of the vectorized mean and covariance with respect to $\batchinpvars$.

    We limit our discussion to the scalar case for simplicity. However, following~\citet{wang2020parallel} or~\citet{balandat2020botorch}, one can apply a similar analysis to obtain a bound for the acquisition sensitivity $\acqfsensitivity^{(\batchsize)}(\cusvector[bm]{\manifoldparameters}_{\boiteration,\batchinpvars})$ for Multi-Objective Bayesian Optimization or parallel BO with $\batchsize>1$.
\end{remark}

\subsection{Pullback FIM Trace Under the Critical Lengthscale}
\label{apx:proof_hd}

\paragraph{Distance concentration.}
We assume that the initial training points $\cusvector{\inpvars}_i$ for $i = 1, \ldots, \Numobservation$ are drawn i.i.d.\ from $\mathrm{Uniform}[0,1]^{\dims{\inpvars}}$. Let $d:=\|\cusvector{\inpvars} - \cusvector{\inpvars}_{\numobservation}\|$. Since $\mathbb{E}\!\left[(\inpvars_j-\inpvars_{\numobservation,j})^2\right] = (\inpvars_j-\tfrac12)^2 + \tfrac1{12}$ for $\inpvars_{\numobservation,j}\sim\mathrm{Uniform}[0,1]$, we have
\[
    \frac{d}{\sqrt{\dims{\inpvars}}}
    \xrightarrow{\mathrm{a.s.}}
    \sqrt{v_{\cusvector{\inpvars}}},
    \qquad
    v_{\cusvector{\inpvars}}
    := \frac{1}{\dims{\inpvars}}\sum_{j=1}^{\dims{\inpvars}}
    \left[(\inpvars_j-\tfrac12)^2+\tfrac1{12}\right]
    \in \left[\tfrac{1}{12},\,\tfrac{1}{3}\right].
\]
Hence $d = \Theta(\sqrt{\dims{\inpvars}})$ almost surely. Here $f = \Theta(g)$ means there exist $c_1, c_2 > 0$ and $d_0$ such that $c_1 g(d) \leq f(d) \leq c_2 g(d)$ for all $d \geq d_0$.

\paragraph{FIM trace for the SE kernel.}
For a GP with SE kernel $\gpkernel_\lengthscale(\cusvector{\inpvars}, \cusvector{\inpvars}') = \signalvar\exp\!\bigl(-\|\cusvector{\inpvars}-\cusvector{\inpvars}'\|^2/(2\lengthscale^2)\bigr)$, let $\GPkernel_{\numobservation\numobservation'} = \gpkernel_\lengthscale(\cusvector{\inpvars}_{\numobservation},\cusvector{\inpvars}_{\numobservation'})$ and $\gpweights = (\GPkernel + \noisevar I)^{-1}\cusvector{\obsouts}$. Then
\[
    \nabla_{\cusvector{\inpvars}}\distmean_{\boiteration}(\cusvector{\inpvars})
    = \sum_{\numobservation=1}^{\Numobservation} \alpha_{\numobservation}\,\nabla_{\cusvector{\inpvars}}\gpkernel_\lengthscale(\cusvector{\inpvars}, \cusvector{\inpvars}_{\numobservation}),
    \qquad
    \nabla_{\cusvector{\inpvars}}\gpkernel_\lengthscale(\cusvector{\inpvars}, \cusvector{\inpvars}_{\numobservation})
    = -\frac{\gpkernel_\lengthscale(\cusvector{\inpvars},\cusvector{\inpvars}_{\numobservation})}{\lengthscale^2}\,(\cusvector{\inpvars} - \cusvector{\inpvars}_{\numobservation}).
\]
By Corollary~\ref{cor:gaussian},
\[
    \mathrm{tr}(\cusvector{\pullbackFIM}_{\boiteration})
    = \frac{\|\nabla\distmean_{\boiteration}\|^2}{\distcov_{\boiteration}}
    + \frac{\|\nabla\distcov_{\boiteration}\|^2}{2\distcov_{\boiteration}^2},
\]
and both terms scale with $\gpkernel_\lengthscale/\lengthscale^2$. Substituting $\lengthscale = c\dims{\inpvars}^\lengthscaleexponent$ and $d = \Theta(\sqrt{\dims{\inpvars}})$ gives
\[
    \gpkernel_\lengthscale(\cusvector{\inpvars}, \cusvector{\inpvars}_{\numobservation})
    \;\xrightarrow{\mathrm{a.s.}}\;
    \signalvar\exp\!\left(-\frac{v_{\cusvector{\inpvars}}}{2c^2}\,\dims{\inpvars}^{1-2\lengthscaleexponent}\right).
\]
This results in three cases:

\par\smallskip\noindent\textbf{Case $\lengthscaleexponent < 1/2$.}
The exponent $v_{\cusvector{\inpvars}}\dims{\inpvars}^{1-2\lengthscaleexponent}/(2c^2) \to \infty$, so $\gpkernel_\lengthscale \to 0$ exponentially. Hence $\nabla_{\cusvector{\inpvars}}\gpkernel_\lengthscale \to 0$ exponentially, and therefore $\nabla\distmean_{\boiteration} \to 0$ and $\mathrm{tr}(\cusvector{\pullbackFIM}_{\boiteration}) \to 0$ exponentially.

\par\smallskip\noindent\textbf{Case $\lengthscaleexponent > 1/2$.}
The exponent $\to 0$, so $\gpkernel_\lengthscale \to \signalvar$ (constant). The gradient $\nabla_{\cusvector{\inpvars}}\gpkernel_\lengthscale \propto \gpkernel_\lengthscale \cdot (\cusvector{\inpvars} - \cusvector{\inpvars}_{\numobservation})/\lengthscale^2$. Since each coordinate difference is $\Theta(1)$ while $\lengthscale^2 = c^2\dims{\inpvars}^{2\lengthscaleexponent}$, each entry scales as $\Theta(\dims{\inpvars}^{-2\lengthscaleexponent})$ and the full gradient norm scales as $\Theta(\dims{\inpvars}^{1/2-2\lengthscaleexponent}) \to 0$. Hence $\nabla\distmean_{\boiteration} \to 0$ and $\mathrm{tr}(\cusvector{\pullbackFIM}_{\boiteration}) \to 0$ at a polynomial rate.

\par\smallskip\noindent\textbf{Case $\lengthscaleexponent = 1/2$.}
The exponent converges to $v_{\cusvector{\inpvars}}/(2c^2) \in (0,\infty)$, so $\gpkernel_\lengthscale \to \gamma_c := \signalvar e^{-v_{\cusvector{\inpvars}}/(2c^2)} \in (0,\signalvar)$. The Gram matrix concentrates to $\GPkernel_\infty = (\signalvar + \noisevar - \bar\gamma)I + \bar\gamma\,\mathbf{1}\mathbf{1}^\top$, where $\bar\gamma = \signalvar e^{-1/(12c^2)}$ is the pairwise limit for i.i.d.\ training pairs. Since $\lambda_{\min}(\GPkernel_\infty) > 0$, we have $\gpweights = \Theta(1)$. Each entry of $\nabla_{\cusvector{\inpvars}}\gpkernel_\lengthscale$ scales as $\Theta(\dims{\inpvars}^{-1})$, so the full gradient norm scales as $\Theta(\dims{\inpvars}^{-1/2})$ and summing $\dims{\inpvars}$ squared terms gives $\mathrm{tr}(\cusvector{\pullbackFIM}_{\boiteration}) = \Theta(\dims{\inpvars}^{-1})$.

In all three cases, the trace of $\cusvector{\pullbackFIM}_{\boiteration}$ is governed solely by whether $\gpkernel_\lengthscale$ is able to model meaningful correlation or not. $\lengthscale = \Theta(\sqrt{\dims{\inpvars}})$ is the unique scaling for which $\mathrm{tr}(\cusvector{\pullbackFIM}_{\boiteration})$ stays at the critical $\Theta(\dims{\inpvars}^{-1})$ order instead of decaying faster.

\paragraph{FIM trace at $\lengthscaleexponent = 1/2$ and acquisition gradient bound.}
With $\gpweights = \Theta(1)$ and $\gamma_c = \Theta(1)$, substituting into Proposition~\ref{prop:universal_bound} gives
\[
    \|\nabla_{\cusvector{\inpvars}}\acqf\|^2 \;\leq\; \acqfsensitivity \cdot \Theta(\dims{\inpvars}^{-1}).
\]
Since $\acqfsensitivity$ is uniformly bounded in $\cusvector{\inpvars}$ and $\dims{\inpvars}$ for standard acquisitions (Appendix~\ref{apx:proof_gaussian}), the universal bound scales as $\Theta(\dims{\inpvars}^{-1})$ rather than collapsing exponentially. For $\lengthscaleexponent < 1/2$, the trace collapses exponentially, while for $\lengthscaleexponent > 1/2$ it decays polynomially. In both off-critical regimes, the AF gradient norm still vanishes.

\paragraph{Mat\'{e}rn kernels.}
The same argument applies to isotropic Mat\'{e}rn-$\nu$ kernel, since it depends on inputs only through $d = \|\cusvector{\inpvars} - \cusvector{\inpvars}'\|/\lengthscale$ and the same concentration gives $d \to \sqrt{v_{\cusvector{\inpvars}}\dims{\inpvars}}/\lengthscale$. Thus the non-degenerate scaling remains $\lengthscale = \Theta(\sqrt{\dims{\inpvars}})$. The difference appears when $\lengthscale \ll \sqrt{\dims{\inpvars}}$: using $\gpkernel_\nu(d) \sim d^{\nu-1/2}e^{-\sqrt{2\nu}\,d}$ as $d\to\infty$, the trace decays as $\exp(-c_\nu\sqrt{\dims{\inpvars}}/\lengthscale)$, whereas for the SE kernel it decays as $\exp(-c_1\dims{\inpvars}/\lengthscale^2)$. Since $\dims{\inpvars}/\lengthscale^2 \gg \sqrt{\dims{\inpvars}}/\lengthscale$ in that regime, the SE trace collapses strictly faster, matching the empirical and theoretical results documented by~\citet{xu2024standard}.

\subsection{Connection between TuRBO and FITR}
\label{apx:proof_turbo_fitr}

Consider a GP with SE kernel with Automatic Relevance Determination (ARD) $\gpkernel_\lengthscale(\cusvector{\inpvars}, \cusvector{\inpvars}') = \signalvar\exp\!\bigl(-\sum_j (\inpvars_j - \inpvars'_j)^2/(2\lengthscale_j^2)\bigr)$ evaluated at the trust-region center $\cusvector{\inpvars}_c$.

At a local optimum of $\distmean_{\boiteration}$, $\nabla\distmean_{\boiteration}(\cusvector{\inpvars}_c) \approx 0$, so the mean-gradient term in the pullback FIM (Corollary~\ref{cor:gaussian}) vanishes. The diagonal of $\pullbackFIM_{\boiteration}(\cusvector{\inpvars}_c)$ is then dominated by the variance-gradient term:
\[
    [\pullbackFIM_{\boiteration}(\cusvector{\inpvars}_c)]_{jj}
    \approx \frac{1}{2\distcov_{\boiteration}(\cusvector{\inpvars}_c)^2}
    \left(\frac{\partial \distcov_{\boiteration}}{\partial \inpvars_j}\right)^2.
\]
For the SE kernel,
$\partial \gpkernel_\lengthscale(\cusvector{\inpvars}, \cusvector{\inpvars}') / \partial \inpvars_j
    = -(\inpvars_j-\inpvars_j')\gpkernel_\lengthscale(\cusvector{\inpvars}, \cusvector{\inpvars}')/\lengthscale_j^2$,
so $\partial \distcov_{\boiteration} / \partial \inpvars_j$ scales proportionally to $\lengthscale_j^{-2}$. Using the definition of the pullback tensor, the diagonal entry $[\pullbackFIM_{\boiteration}(\cusvector{\inpvars}_c)]_{jj}$ then scales like $\lengthscale_j^{-4}$. The ellipsoid induced by the pullback tensor is therefore longer in dimensions with longer lengthscales and shorter in dimensions with shorter lengthscales:
\[
    \{(\cusvector{\inpvars} - \cusvector{\inpvars}_c)^\top \mathrm{diag}(\pullbackFIM_{\boiteration}(\cusvector{\inpvars}_c))(\cusvector{\inpvars} - \cusvector{\inpvars}_c) \leq \delta^2\}.
\]
The radius in dimension $j$ is proportional to $[\pullbackFIM_{\boiteration}]_{jj}^{-1/2}$. TuRBO's trust region $\{|\inpvars_j - \inpvars_{c,j}| \leq L\lengthscale_j\}$ is not an exact diagonal-FIM box, but it preserves the same qualitative ordering: directions with larger ARD lengthscales receive wider search lengths. This motivates extending TuRBO to a location-dependent trust region whose scaling is determined by predictive sensitivities rather than explicit kernel lengthscales.
\subsection{Local KL Expansion for Pullback-Fisher Trust Regions}
\label{apx:kl_trust_region}

This subsection gives the local probabilistic interpretation of the trust-region ellipsoid from Section~\ref{sec:fitr}. The key point is that the full pullback Fisher tensor defines the second-order term of the KL divergence between nearby predictive posteriors, so a trust region based on this tensor limits local predictive change rather than raw Euclidean displacement.

Fix an iteration $\boiteration$ and center $\cusvector{\inpvars}_{\boiteration,c}\in\inpvarsset$. Let
\[
    \cusvector[bm]{\manifoldparameters}_{\boiteration,c}
    :=
    \manifoldmap_{\boiteration}(\cusvector{\inpvars}_{\boiteration,c}),
    \qquad
    \delta := \cusvector{\inpvars}-\cusvector{\inpvars}_{\boiteration,c}.
\]
Consider the KL divergence between the predictive posterior at the center and the predictive posterior at a nearby input:
\begin{equation}
    \label{eq:kl_local_start}
    \mathcal{K}_{\boiteration}(\delta)
    :=
    D_{\mathrm{KL}}\!\left(
    \probability(\obsouts;\cusvector[bm]{\manifoldparameters}_{\boiteration,c})
    \,\middle\|\,
    \probability(\obsouts;\manifoldmap_{\boiteration}(\cusvector{\inpvars}_{\boiteration,c}+\delta))
    \right).
\end{equation}

\paragraph{Step 1: KL expansion in parameter space.}
For a smooth parametric family $\probability(\obsouts;\cusvector[bm]{\manifoldparameters})$, the KL divergence admits the standard second-order expansion around $\cusvector[bm]{\manifoldparameters}_{\boiteration,c}$~\citep{amari2000methods, amari2016information}:
\begin{equation}
    \label{eq:kl_theta_expansion}
    D_{\mathrm{KL}}\!\left(
    \probability(\obsouts;\cusvector[bm]{\manifoldparameters}_{\boiteration,c})
    \,\middle\|\,
    \probability(\obsouts;\cusvector[bm]{\manifoldparameters}_{\boiteration,c}+\Delta\cusvector[bm]{\manifoldparameters})
    \right)
    =
    \tfrac{1}{2}
    \Delta\cusvector[bm]{\manifoldparameters}^{\top}
    \fishermetric_{\manifold}(\cusvector[bm]{\manifoldparameters}_{\boiteration,c})
    \Delta\cusvector[bm]{\manifoldparameters}
    + o(\|\Delta\cusvector[bm]{\manifoldparameters}\|^2).
\end{equation}

\paragraph{Step 2: Pull back to input space.}
Since $\manifoldmap_{\boiteration}$ is $C^1$ by Assumption~\ref{ass:surrogate_regularity}, we have the first-order expansion
\begin{equation}
    \label{eq:posterior_map_taylor}
    \manifoldmap_{\boiteration}(\cusvector{\inpvars}_{\boiteration,c}+\delta)
    =
    \cusvector[bm]{\manifoldparameters}_{\boiteration,c}
    +
    \jacobianmap_{\boiteration}(\cusvector{\inpvars}_{\boiteration,c})\,\delta
    +
    o(\|\delta\|).
\end{equation}
Substituting
\[
    \Delta\cusvector[bm]{\manifoldparameters}
    =
    \jacobianmap_{\boiteration}(\cusvector{\inpvars}_{\boiteration,c})\,\delta
    +
    o(\|\delta\|)
\]
into~\eqref{eq:kl_theta_expansion} gives
\begin{align}
    \mathcal{K}_{\boiteration}(\delta)
    &=
    \tfrac{1}{2}
    \delta^\top
    \jacobianmap_{\boiteration}(\cusvector{\inpvars}_{\boiteration,c})^\top
    \fishermetric_{\manifold}(\cusvector[bm]{\manifoldparameters}_{\boiteration,c})
    \jacobianmap_{\boiteration}(\cusvector{\inpvars}_{\boiteration,c})
    \delta
    + o(\|\delta\|^2) \nonumber \\
    &=
    \tfrac{1}{2}
    \delta^\top
    \cusvector{\pullbackFIM}_{\boiteration}(\cusvector{\inpvars}_{\boiteration,c})
    \delta
    + o(\|\delta\|^2),
    \label{eq:kl_pullback_expansion}
\end{align}
where the second line uses Definition~\ref{def:pullback_tensor}.

Equation~\eqref{eq:kl_pullback_expansion} shows that the ellipsoid
\begin{equation}
    \label{eq:kl_pullback_ball}
    \delta^\top \cusvector{\pullbackFIM}_{\boiteration}(\cusvector{\inpvars}_{\boiteration,c}) \delta
    \leq \TRlength_{\boiteration}^2
\end{equation}
is, to second order, a local KL trust region around the predictive posterior at $\cusvector{\inpvars}_{\boiteration,c}$. This interpretation applies to the full pullback Fisher tensor. The FITR implementation in Section~\ref{sec:fitr} uses a diagonal, regularized approximation of~\eqref{eq:kl_pullback_ball} to obtain an axis-aligned box that is computationally compatible with TuRBO-style trust-region updates.

\subsection{Pullback FIM Diagonal via Backpropagation}
\label{apx:fitr_estimator}

The estimator~\eqref{eq:fitr_diag_estimate} computes the diagonal of the pullback Fisher tensor by differentiating the log predictive likelihood with respect to the input $\cusvector{\inpvars}$ rather than the distribution parameters $\cusvector[bm]{\manifoldparameters}$. The connection follows from the chain rule and the definition of the FIM.

By the chain rule, the partial derivative of the log predictive likelihood with respect to input coordinate $\inpvars_j$ is
\begin{equation}
    \label{eq:log_lik_chain_rule}
    \frac{\partial}{\partial \inpvars_j}
    \log \probability\bigl(\obsouts;\manifoldmap_{\boiteration}(\cusvector{\inpvars})\bigr)
    =
    \bigl(\nabla_{\cusvector[bm]{\manifoldparameters}}\log \probability(\obsouts;\cusvector[bm]{\manifoldparameters})\bigr)^\top
    \frac{\partial\manifoldmap_{\boiteration}(\cusvector{\inpvars})}{\partial \inpvars_j}.
\end{equation}
Squaring both sides and taking the expectation over $\obsouts\sim\probability(\obsouts;\manifoldmap_{\boiteration}(\cusvector{\inpvars}))$,
\begin{equation}
    \mathbb{E}\!\left[
        \left(\frac{\partial}{\partial \inpvars_j}
        \log \probability\bigl(\obsouts;\manifoldmap_{\boiteration}(\cusvector{\inpvars})\bigr)\right)^{\!2}
        \right]
    =
    \left(\frac{\partial\manifoldmap_{\boiteration}}{\partial \inpvars_j}\right)^\top
    \underbrace{\mathbb{E}\!\left[
            \nabla_{\cusvector[bm]{\manifoldparameters}}\log \probability\,
            \nabla_{\cusvector[bm]{\manifoldparameters}}\log \probability^\top
            \right]}_{=\,\fishermetric_{\manifold}(\manifoldmap_{\boiteration}(\cusvector{\inpvars}))}
    \left(\frac{\partial\manifoldmap_{\boiteration}}{\partial \inpvars_j}\right)
    =
    \bigl[\cusvector{\pullbackFIM}_{\boiteration}(\cusvector{\inpvars})\bigr]_{jj},
\end{equation}
where the last equality uses the definition of the FIM~\eqref{eq:fim_def} and the pullback tensor~\eqref{eq:pullback_tensor}. The diagonal entry of the pullback Fisher tensor is therefore the expected squared partial derivative of the log predictive likelihood with respect to $\inpvars_j$, computed at the predictive posterior $\manifoldmap_{\boiteration}(\cusvector{\inpvars})$.

Replacing the expectation by an average over $\MCsamples$ draws $\tilde{\obsouts}_\mcsamples \sim \probability(\obsouts;\manifoldmap_{\boiteration}(\cusvector{\inpvars}))$ and evaluating the partial derivatives by backpropagation through the log-likelihood gives the estimator~\eqref{eq:fitr_diag_estimate}. This is the outer-product-of-gradients form of the FIM, similar to~\citep{martens2020new}, applied to the input coordinates $\cusvector{\inpvars}$ instead of the distribution parameters $\cusvector[bm]{\manifoldparameters}$.

% \input{appendix/proof_spherical_mapping.tex}
% \input{appendix/proof_lq_kernel_scaling.tex}

%%%%%%%%%%%%%%%%%%%%%%%%%%%%%%%%%%%%%%%%%%%%%%%%%%%%%%%%%%%%

\newpage
\section*{NeurIPS Paper Checklist}

\begin{enumerate}

    \item {\bf Claims}
    \item[] Question: Do the main claims made in the abstract and introduction accurately reflect the paper's contributions and scope?
    \item[] Answer: \answerYes{} % Replace by \answerYes{}, \answerNo{}, or \answerNA{}.
    \item[] Justification: The abstract and introduction clearly state the claims made, including the contributions made in the paper. The Figure 1 and lines 41-52 in the introduction section and lines 2-3, 6-10 in the abstract specify the main claims of the paper.

    \item {\bf Limitations}
    \item[] Question: Does the paper discuss the limitations of the work performed by the authors?
    \item[] Answer: \answerYes{} % Replace by \answerYes{}, \answerNo{}, or \answerNA{}.
    \item[] Justification: The paper discusses the limitations of the work in the section 5 and also in the Limitations section.

    \item {\bf Theory assumptions and proofs}
    \item[] Question: For each theoretical result, does the paper provide the full set of assumptions and a complete (and correct) proof?
    \item[] Answer: \answerYes{} % Replace by \answerYes{}, \answerNo{}, or \answerNA{}.
    \item[] Justification: The paper provides the full set of assumptions and a complete proof for the theoretical results. We provide a proof for the main theorem in the appendix.

    \item {\bf Experimental result reproducibility}
    \item[] Question: Does the paper fully disclose all the information needed to reproduce the main experimental results of the paper to the extent that it affects the main claims and/or conclusions of the paper (regardless of whether the code and data are provided or not)?
    \item[] Answer: \answerYes{} % Replace by \answerYes{}, \answerNo{}, or \answerNA{}.
    \item[] Justification: The pseudocode and the experiment setup fully discloses all the information needed to reproduce the main experimental results of the paper. The used benchmark datasets are publicly available and very often used in the BO literature.

    \item {\bf Open access to data and code}
    \item[] Question: Does the paper provide open access to the data and code, with sufficient instructions to faithfully reproduce the main experimental results, as described in supplemental material?
    \item[] Answer: \answerNo{} % Replace by \answerYes{}, \answerNo{}, or \answerNA{}.
    \item[] Justification: The code is currently undergoing institutional legal review and will be released open-source (with the repository linked on arXiv) once cleared. All experiments rely entirely on open-access benchmarks and the open-source library BoTorch. Furthermore, the proposed methods, bounds, and experimental setups can be fully reproduced using the detailed pseudocode and implementation specifications provided in the paper.

    \item {\bf Experimental setting/details}
    \item[] Question: Does the paper specify all the training and test details (e.g., data splits, hyperparameters, how they were chosen, type of optimizer) necessary to understand the results?
    \item[] Answer: \answerYes{} % Replace by \answerYes{}, \answerNo{}, or \answerNA{}.
    \item[] Justification: The paper specifies all the training and test details (e.g., data splits, hyperparameters, how they were chosen, type of optimizer) necessary to understand the results. The experiment used the BoTorch package primarily. The default hyperparameters and training routines from BoTorch were used for the surrogate model training and AF optimization. The hyperparameters that were changed, have been specified in the paper.

    \item {\bf Experiment statistical significance}
    \item[] Question: Does the paper report error bars suitably and correctly defined or other appropriate information about the statistical significance of the experiments?
    \item[] Answer: \answerYes{} % Replace by \answerYes{}, \answerNo{}, or \answerNA{}.
    \item[] Justification: The experiment plots show an average and min-max shaded region for the the different methods. The shaded region is calculated as the standard error from 11 runs for each method.

    \item {\bf Experiments compute resources}
    \item[] Question: For each experiment, does the paper provide sufficient information on the computer resources (type of compute workers, memory, time of execution) needed to reproduce the experiments?
    \item[] Answer: \answerYes{} % Replace by \answerYes{}, \answerNo{}, or \answerNA{}.
    \item[] Justification: The experiments were run on an Nvidia GPU Tesla V100-SXM2-32GB. The experiments were run for 11 independent runs for each method. The BO budget for each iteration is provided in the experiment plots.

    \item {\bf Code of ethics}
    \item[] Question: Does the research conducted in the paper conform, in every respect, with the NeurIPS Code of Ethics \url{https://neurips.cc/public/EthicsGuidelines}?
    \item[] Answer: \answerYes{} % Replace by \answerYes{}, \answerNo{}, or \answerNA{}.
    \item[] Justification: The research conducted in the paper conforms, in every respect, with the NeurIPS Code of Ethics \url{https://neurips.cc/public/EthicsGuidelines}. The paper does not involve any human subjects or any sensitive data. All the experiments were run on a publicly available dataset and code was developed using open-source libraries.

    \item {\bf Broader impacts}
    \item[] Question: Does the paper discuss both potential positive societal impacts and negative societal impacts of the work performed?
    \item[] Answer: \answerNA{} % Replace by \answerYes{}, \answerNo{}, or \answerNA{}.
    \item[] Justification: This work discusses a theoretical aspect of BO and is useful for black-box optimization in general. However, we do not see any broader societal impact of this work.

    \item {\bf Safeguards}
    \item[] Question: Does the paper describe safeguards that have been put in place for responsible release of data or models that have a high risk for misuse (e.g., pre-trained language models, image generators, or scraped datasets)?
    \item[] Answer: \answerNA{} % Replace by \answerYes{}, \answerNo{}, or \answerNA{}.
    \item[] Justification: This work does not release any models or datasets that have a high risk for misuse.

    \item {\bf Licenses for existing assets}
    \item[] Question: Are the creators or original owners of assets (e.g., code, data, models), used in the paper, properly credited and are the license and terms of use explicitly mentioned and properly respected?
    \item[] Answer: \answerYes{} % Replace by \answerYes{}, \answerNo{}, or \answerNA{}.
    \item[] Justification: The paper properly credits the papers, opensource libraries and datasets used in the paper. The credit to the authors of code is explicitly mentioned in the code repository as well.

    \item {\bf New assets}
    \item[] Question: Are new assets introduced in the paper well documented and is the documentation provided alongside the assets?
    \item[] Answer: \answerNA{} % Replace by \answerYes{}, \answerNo{}, or \answerNA{}.
    \item[] Justification: This work does not introduce any new assets.

    \item {\bf Crowdsourcing and research with human subjects}
    \item[] Question: For crowdsourcing experiments and research with human subjects, does the paper include the full text of instructions given to participants and screenshots, if applicable, as well as details about compensation (if any)?
    \item[] Answer: \answerNA{} % Replace by \answerYes{}, \answerNo{}, or \answerNA{}.
    \item[] Justification: This work does not involve crowdsourcing nor research with human subjects.

    \item {\bf Institutional review board (IRB) approvals or equivalent for research with human subjects}
    \item[] Question: Does the paper describe potential risks incurred by study participants, whether such risks were disclosed to the subjects, and whether Institutional Review Board (IRB) approvals (or an equivalent approval/review based on the requirements of your country or institution) were obtained?
    \item[] Answer: \answerNA{} % Replace by \answerYes{}, \answerNo{}, or \answerNA{}.
    \item[] Justification: This work does not involve research with human subjects.

    \item {\bf Declaration of LLM usage}
    \item[] Question: Does the paper describe the usage of LLMs if it is an important, original, or non-standard component of the core methods in this research? Note that if the LLM is used only for writing, editing, or formatting purposes and does \emph{not} impact the core methodology, scientific rigor, or originality of the research, declaration is not required.
          %this research? 
    \item[] Answer: \answerNA{} % Replace by \answerYes{}, \answerNo{}, or \answerNA{}.
    \item[] Justification: This work does not involve LLMs as any important, original, or non-standard components.

\end{enumerate}

\end{document}